\documentclass{new_tlp}
\usepackage{latexsym}
\usepackage{color}
\usepackage{oldlfont,rawfonts}
\usepackage{graphicx}
\usepackage{verbatim}

\usepackage{amssymb}
\usepackage{amsmath}

\newcommand {\lingconc} {\mathcal{S}}

\newcommand {\ent} {\mathrel{{\scriptstyle\mid\!\sim}}}

\newcommand {\sx} {\langle}
\newcommand {\dx} {\rangle}

\newcommand {\emme} {{\mathcal{M}}}
\newcommand {\enne} {\mathcal{N}}

\newcommand {\tc} {\mid}
\newcommand {\vuoto} {\emptyset}

\newcommand{\tip}{{\bf T}}

\newcommand{\sroel}{\mathcal{SROEL}(\sqcap,\times)}
\newcommand{\alc}{\mathcal{ALC}}
\newcommand{\shiq}{\mathcal{SHIQ}}

\newcommand{\alctmin}{{\mathcal{ALC}}+\tip_{min}}
\newcommand{\el}{\mathcal{EL}^{\bot}}

\newcommand{\elpb}{{\mathcal{EL}}^{+}_{\bot}}

\newcommand{\eltm}{{\mathcal{EL}^{\bot} \tip_{\mathit{min}}}}

\newcommand{\alctr}{\mathcal{ALC}+\tip_{\textsf{\tiny R}}}

\newcommand{\dlltm}{{\mbox{\em DL-Lite}_{\mathit{c}}\tip_{\mathit{min}}}}

\newcommand{\be}{\begin{enumerate}}
\newcommand{\ee}{\end{enumerate}}

\newcommand{\hide}[1]{}

\def \cases{\left \{\begin{array}{l}}
\def \endcases{\end{array}\right .}

\newcommand {\bes} {\begin{description}}
\newcommand{\ens} {\end{description}}

\newcommand {\beq} {\begin{quote}}
\newcommand {\enq} {\end{quote}}
\newcommand {\bit} {\begin{itemize}}
\newcommand {\enit} {\end{itemize}}

\newcommand{\wARC}{w^{{A \sqsubseteq \exists R.C}}}

\newcommand{\auxARC}{{aux^{{A \sqsubseteq \exists R.C}}}}

\newcommand{\prj}{{\iota}}

\newenvironment{pozz}{\color{black}}{\color{black}}

\newtheorem{lemma}{Lemma}

\newtheorem{proposition}{Proposition}
\newtheorem{definition}{Definition}
\newtheorem{example}{Example}

 \title[Theory and Practice of Logic Programming]
 { 
 	An ASP approach for reasoning in \\ a concept-aware multipreferential lightweight DL
  }

\author[L. Giordano, D. Theseider Dupr{\'e} ]
{Laura Giordano,  Daniele Theseider Dupr{\'e}  \\
DISIT, Universit\`a del Piemonte Orientale, Italy\\
\email{laura.giordano@uniupo.it, dtd@uniupo.it} \\
}

\jdate{March 2003}
\pubyear{2003}
\pagerange{\pageref{firstpage}--\pageref{lastpage}}
\doi{S1471068401001193}

\begin{document}
\bibliographystyle{acmtrans}

\label{firstpage}

\maketitle
 
 \begin{abstract} 
In this paper we develop a concept aware multi-preferential semantics for dealing with typicality in  description logics, where preferences are associated with concepts, starting from a collection of ranked TBoxes containing defeasible concept inclusions.
Preferences are combined to define a preferential interpretation in which  defeasible inclusions can be evaluated. 
The construction of the concept-aware multipreference semantics is related to Brewka's framework for qualitative preferences. 
We exploit Answer Set Programming (in particular, {\em asprin}) to achieve defeasible reasoning under the multipreference approach for the lightweight description logic $\elpb$. 
\smallskip

\noindent The paper is under consideration for acceptance in TPLP.

\end{abstract}

  \begin{keywords}
  Nonmonotonic Reasoning, Description Logics, Preferences, ASP 
  \end{keywords}

\section{Introduction}

The need to reason about exceptions in ontologies has led to the development of many non-monotonic extensions of Description Logics (DLs), 
incorporating features from NMR formalisms in the literature 
\cite{Straccia93,baader95a,donini2002,lpar2007,sudafricaniKR,bonattilutz,casinistraccia2010,rosatiacm}, 
and notably including extensions of rule-based languages \cite{Eiter2008,Eiter2011,KnorrECAI12,Gottlob14,TPLP2016,Bozzato2018},
as well as  new constructions and semantics \cite{CasiniJAIR2013,bonattiAIJ15,Bonatti2019}. Preferential approaches  \cite{KrausLehmannMagidor:90,whatdoes}
have been extended to description logics, to deal with inheritance with exceptions in ontologies,
allowing for non-strict forms of inclusions,
called {\em typicality or defeasible inclusions}, with different preferential semantics \cite{lpar2007,sudafricaniKR} 
and closure constructions \cite{casinistraccia2010,CasiniDL2013,dl2013,Pensel18}.

In this paper, 
we propose 
a ``concept-aware multipreference semantics" for reasoning about exceptions in ontologies taking into account preferences with respect to different concepts and integrating them into a preferential semantics which allows a standard interpretation of defeasible inclusions.
The intuitive idea is that
the relative typicality of two domain individuals usually depends on the aspects we are considering for comparison:
Bob may be a more typical as sport lover  than Jim, but Jim may be a  more typical swimmer than Bob.
This leads to consider a multipreference semantics in which there is a preference relation $\leq_C$ among individuals for each {\em aspect}  (concept) $C$.
In the previous case, we would have $bob \leq_{SportLover} jim$ and  $jim \leq_{Swimmer} bob$.
Considering different preference relations associated with concepts, and then combining them into a global preference, provides a simple solution to the blocking inheritance problem, which affects rational closure, while still allowing to deal with specificity and irrelevance.

Our approach is strongly related with Gerard Brewka's proposal of preferred subtheories  \cite{Brewka89}, 
later generalized within the framework of Basic Preference Descriptions  for ranked knowledge bases \cite{Brewka04}. 
We extend to DLs the idea of having {\em ranked} or {\em stratified} knowledge bases 
(ranked TBoxes  here) and to define preorders (preferences) on worlds (here, preferences among domain elements in a DL interpretation). 
Furthermore, we associate ranked TBoxes with concepts. 
The ranked TBox for concept $C$ describes the prototypical properties of $C$-elements. For instance, the ranked TBox for concept {\em Horse} describes the typical properties of horses, of  running fast,
having a long mane, being tall, having a tail and a saddle. These properties are defeasible and horses should not necessarily satisfy all of them. 

The  ranked TBox for $C_h$ determines a preference relation $\leq_{C_h}$ 
on the domain,
defining the relative typicality of domain elements with respect to aspect $C_h$.
We  then combine such preferences into a global preference relation $<$ to define a concept-wise multipreference semantics, 
in which {\em all} conditional queries can be evaluated 
as usual in preferential semantics. For instance, 
we may want to check whether typical Italian employees have a boss, given the preference relation $\leq_{\mathit{Employee}}$, 
but no preference relation for concept $\mathit{Italian}$; 
or to check whether
employed students are normally young or have a boss, given the preference relations $\leq_{\mathit{Employee}}$ and $\leq_{\mathit{Student}}$, resp., for employees and for students. 

We introduce a notion of multipreference entailment 
and prove that it satisfies the KLM properties of preferential consequence relations.
This notion of entailment deals properly with irrelevance and specificity, is not subject to the ``blockage of property inheritance" problem,  which affects rational closure \cite{PearlTARK90}, 
i.e., if a subclass is exceptional with respect to a superclass for a given property, it does not inherit from that superclass any other property.

To prove the feasibility of our approach, 
we develop a proof method for reasoning under the proposed multipreference semantics for the description logic $\elpb$ \cite{IncredibleELK_JAR14},
the fragment of OWL2 EL supported by ELK.
We reformulate multipreference entailment as a problem of computing preferred answer sets and,
as a natural choice, we develop an encoding of the multipreferential extension of $\elpb$ in \emph{asprin} \cite{BrewkaAAAI15}, 
exploiting a fragment of Kr\"{o}tzsch's Datalog materialization calculus \shortcite{KrotzschJelia2010}.
 
As a consequence of the soundness and completeness of 
this  reformulation of multipreference entailment,  
we prove that 
concept-wise multipreference entailment is $\Pi^p_2$-complete for  $\elpb$ ranked knowledge bases.

\section{Preliminaries: The description logics $\elpb$}\label{sec:ALC}

We consider the description logic $\elpb$ \cite{IncredibleELK_JAR14} of the ${\cal EL}$ family  \cite{rifel}. 
Let ${N_C}$ be a set of concept names, ${N_R}$ a set of role names
  and ${N_I}$ a set of individual names.  
The set  of $\elpb$ \emph{concepts} can be
defined as follows: 
$C \ \ := A \tc \top \tc \bot  \tc C \sqcap C \tc \exists r.C $, 
where $a \in N_I$, $A \in N_C$ and $r \in N_R$. 
Observe that union, complement and universal restriction are not $\elpb$ constructs.
A knowledge base (KB) $K$ is a pair $({\cal T}, {\cal A})$, where ${\cal T}$ is a TBox and
${\cal A}$ is an ABox.
The TBox ${\cal T}$ is  a set of {\em concept inclusions} (or subsumptions) of the form $C \sqsubseteq D$, where $C,D$ are concepts,
and of  {\em role inclusions}  of the form $r_1 \circ \ldots \circ r_n \sqsubseteq r$, where $r_1,\ldots,r_n, r \in N_R$.
The  ABox ${\cal A}$ is  a set of assertions of the form $C(a)$ 
and $r(a,b)$ where $C$ is a  concept, $r \in N_R$, and $a, b \in N_I$.

An {\em interpretation} for $\elpb$ is a pair $I=\langle \Delta, \cdot^I \rangle$ where:
$\Delta$ is a non-empty domain---a set whose elements are denoted by $x, y, z, \dots$---and 
$\cdot^I$ is an extension function that maps each
concept name $C\in N_C$ to a set $C^I \subseteq  \Delta$, each role name $r \in N_R$
to  a binary relation $r^I \subseteq  \Delta \times  \Delta$,
and each individual name $a\in N_I$ to an element $a^I \in  \Delta$.
It is extended to complex concepts  as follows:
$\top^I=\Delta$, $\bot^I=\vuoto$, 
$(C \sqcap D)^I =C^I \cap D^I$  and 
$(\exists r.C)^I =\{x \in \Delta \tc \exists y.(x,y) \in r^I \ \mbox{and} \ y \in C^I\}.$	
The notions of satisfiability of a KB  in an interpretation and of entailment are defined as usual:

\begin{definition}[Satisfiability and entailment] \label{satisfiability}
Given an $\elpb$ interpretation $I=\langle \Delta, \cdot^I \rangle$: 

	- $I$  satisfies an inclusion $C \sqsubseteq D$ if   $C^I \subseteq D^I$;
	
	- $I$  satisfies a role inclusions  $r_1 \circ \ldots \circ r_n \sqsubseteq r$ if  $r_1^I \circ \ldots \circ r_n^I \subseteq r^I$;

	-  $I$ satisfies an assertion $C(a)$ if $a^I \in C^I$ and an assertion $r(a,b)$ if $(a^I,b^I) \in r^I$.

\noindent
 Given  a KB $K=({\cal T}, {\cal A})$,
 an interpretation $I$  satisfies ${\cal T}$ (resp. ${\cal A}$) if $I$ satisfies all  inclusions in ${\cal T}$ (resp. all assertions in ${\cal A}$);
 $I$ is a \emph{model} of $K$ if $I$ satisfies ${\cal T}$ and ${\cal A}$.

 A subsumption $F= C \sqsubseteq D$ (resp., an assertion $C(a)$, $R(a,b)$),   {is entailed by $K$}, written $K \models F$, if for all models $I=$$\sx \Delta,  \cdot^I\dx$ of $K$,
$I$ satisfies $F$.
\end{definition}

\section{Multiple preferences from ranked TBoxes}\label{sez:RC}

To define a multipreferential semantics for $\elpb$ we extend the language with a typicality operator $\tip$, as done for $\el$ \cite{ijcai2011}. 
In the   language extended  with the typicality operator, an additional concept $\tip(C)$ is allowed (where $C$ is an $\elpb$ concept), whose instances are intended to be the {\em prototypical} instances of concept $C$.
Here, we assume that $\tip(C)$ can only occur on the left hand side of concept inclusion, to allow typicality inclusions of the form $\tip(C) \sqsubseteq D$, meaning that ``typical C's are D's" or ``normally C's are D's". Such inclusions are defeasible, i.e.,  admit exceptions, while standard inclusions are called {\em strict}, and must be satisfied by all domain elements. 

Let ${\cal C}$ be a (finite) set  of distinguished concepts $\{C_1, \ldots, C_k\}$, where $C_1, \ldots, C_k$ are possibly complex $\elpb$ concepts. 
Inspired to Brewka's framework of basic preference descriptions  \cite{Brewka04}, 
we introduce a {\em ranked TBox}  ${\cal T}_{C_i}$ for each concept $C_i \in {\cal C}$,  describing the typical properties $\tip(C_i) \sqsubseteq D$ of $C_i$-elements. 
Ranks (non-negative integers) are assigned to such inclusions; the ones with higher rank are considered more important than the ones with lower rank.

A {\em ranked $\elpb$ knowledge base $K$ over ${\cal C}$} is a tuple $\langle  {\cal T}_{strict}, {\cal T}_{C_1}, \ldots, {\cal T}_{C_k}, {\cal A}  \rangle$, 
where ${\cal T}_{strict}$ is a set of standard concept and role inclusions, ${\cal A}$ is an ABox and, for each $C_j \in {\cal C}$,  ${\cal T}_{C_j}$ is a ranked TBox of defeasible inclusions,
$\{(d^j_i,r^j_i)\}$, where  each  $d^j_i$ is a typicality inclusion of the form $\tip(C_j) \sqsubseteq D^j_i$,  having rank $r^j_i$, a non-negative integer.

\begin{example} \label{exa:horse}
	Consider the ranked KB $\mathit{K= \langle {\cal T}_{strict}, {\cal T}_{Horse}, {\cal A} \rangle}$ (with empty ${\cal A}$), where ${\cal T}_{strict}$ contains 
	$\mathit{Horse \sqsubseteq  Mammal}$, $\mathit{Mammal \sqsubseteq Animal}$,
	and ${\cal T}_{Horse}=\{(d_1,0),(d_2,0),(d_3, 1),(d_4,2)\}$ \linebreak where the defeasible inclusions $d_1,\ldots, d_4$ are as follows:

	$(d_1)$ $\mathit{\tip(Horse) \sqsubseteq \exists has\_equipment. Saddle}$ \ \ \ \ \ \ \ \ \ $(d_2)$ $\mathit{\tip(Horse) \sqsubseteq \exists Has\_Mane.Long}$
	
	$(d_3)$ $\mathit{\tip(Horse) \sqsubseteq RunFast}$  \ \ \ \ \ \ \ \ \  \ \ \ \ \ \ \ \ \  \ \ \ \ \ \ \ \ \ \  $(d_4)$ $\mathit{\tip(Horse) \sqsubseteq \exists Has\_Tail.\top}$
	
	\end{example}
The ranked Tbox ${\cal T}_{Horse}$  can be used to define an ordering among domain elements comparing their typicality as horses.
For instance, given two horses {\em Spirit} and {\em Buddy}, if Spirit has long mane, no saddle, has a tail and  runs fast, it is intended to be more typical than  Buddy, a horse running fast, with saddle and long mane, but without tail, as having a tail (rank 2) is a more important property for horses wrt having a saddle (rank 0).

In order to define an ordering for each $C_i \in {\cal C}$, where $x \leq_{C_i} y$ means that $x$ is at least as typical as $y$ wrt $C_i$  (in the example, $\mathit{Spirit \leq_{Horse} Buddy}$ and, actually, $\mathit{Spirit<_{Horse} Buddy}$), among the preference strategies considered by Brewka, we adopt strategy $\#$,
which considers the number of formulas satisfied by a domain element for each rank.

Given a ranked knowledge base $K=\langle  {\cal T}_{strict}, {\cal T}_{C_1}, \ldots, {\cal T}_{C_k}, {\cal A}  \rangle$, where ${\cal T}_{C_j}=\{(d^j_i,r^j_i)\}$ for all $j=1,\ldots,k$, 
let us consider an $\elpb$ interpretation $I=\langle \Delta, \cdot^I \rangle$ satisfying all the strict inclusions  in ${\cal T}_{strict}$ and assertions in ${\cal A}$.
For each $j$, to define a preference ordering $\leq_{C_j}$ on $\Delta$, 
we first need to determine when a domain element $x \in \Delta$  satisfies/violates 
a typicality inclusion for $C_j$. 
We say that $x \in \Delta$ {\em satisfies} $\tip(C_j) \sqsubseteq D$ in $I$, if  $x   \not \in C_j^I$ or $x \in D^I$, while
$x$ {\em violates} $\tip(C_j) \sqsubseteq D$ in $I$, if  $x   \in C_j^I$ and $x \not  \in D^I$.
Note that any element which is not an instance of $C_j$ trivially satisfies all conditionals $\tip(C_j) \sqsubseteq D^j_i$.
For a domain element $x \in \Delta$, let ${\cal T}_{C_j}^l(x)$ be the set of  typicality inclusions in ${\cal T}_{C_j}$ with rank $l$ which are satisfied by $x$:
${\cal T}_{C_j}^l(x) = \{d \mid (d,l) \in {\cal T}_{C_j} \mbox{ and } x \mbox{ satisfies } d \mbox{ in }  I\}$.

\begin{definition}[ $\leq_{C_j}$] \label{total_preorder}
Given a ranked knowledge base $K$ as above 
and an $\elpb$ interpretation $I=\langle \Delta, \cdot^I \rangle$, 
the {\em preference relation $\leq_{C_j}$  associated with ${\cal T}_{C_j}=\{(D^j_i,r^j_i)\}$ in $I$} is defined as follows:
\begin{align*}
x_1 & \leq_{C_j}  x_2  \mbox{  \ \ iff \ \ }   \mbox{either }  |{\cal T}_{C_j}^l(x_1) |=  |{\cal T}_{C_j}^l(x_2)| \mbox{ for all $l$, } \\
& \mbox{ \ \ \ \ \ or  $\exists l$ such that $|{\cal T}_{C_j}^l(x_1)| > |{\cal T}_{C_j}^l(x_2)|$ and, $\forall h>l$, $|{\cal T}_{C_j}^h(x_1)| = |{\cal T}_{C_j}^h(x_2)|$} 
\end{align*}
A strict preference relation $<_{C_j}$ and an equivalence relation $\sim_{C_j}$ can be defined as usual letting: 
$x_1 <_{C_j} x_2$ iff ($x_1 \leq_{C_j} x_2$ and  not $x_2 \leq_{C_j} x_1$),  and 
$x \sim_{C_j} y$ iff ($x \leq_{C_j} y$ and $y \leq_{C_j} x)$.
\end{definition}
Informally, $\leq_{C_j}$ gives higher preference to domain individuals violating less typicality inclusions with higher rank. 
Definition \ref{total_preorder}  exploits Brewka's  $\#$ strategy  in DL context. In particular,  all  $x,y \not \in C_j^I$, $x \sim_{C_j} y$,
i.e., all elements not belonging to $C_j^I$ are assigned the same rank, the least one, as they trivially satisfy all the typical properties of $C_j$'s.
As, for a ranked knowledge base, the $\#$ strategy defines a total preorder \cite{Brewka04} and, for each  ${\cal T}_{C_j}$, we have applied this strategy to the materializations $C_j \sqsubseteq D$ of the typicality inclusions $\tip(C_j) \sqsubseteq D$ in the ranked TBox  ${\cal T}_{C_j}$, the relation $\leq_{C_j}$ is a total preorder on the domain $\Delta$.
Then, the strict preference relation $<_{C_j}$ is  a strict modular partial order,  i.e., an irreflexive, transitive and modular relation (where modularity means that: 
for all $x,y,z \in \Delta$, if $x <_{C_j} y$ then $x <_{C_j} z$ or $z <_{C_j} y$);  $\sim_{C_j}$ is an equivalence relation.

As $\elpb$ has the {\em finite model property} \cite{rifel}, we can restrict our consideration to interpretations $I$ with a finite domain.
In principle, we would like to consider, for each concept $C_j \in {\cal C}$, {\em all possible domain elements} compatible with the inclusions in $ {\cal T}_{strict}$, and compare them according to $\leq_{C_i}$ relation.
This leads us to restrict the consideration to models of  $ {\cal T}_{strict}$ that we call {\em canonical}, in analogy with the canonical models of rational closure \cite{dl2013}. 
For each concept $C$ occurring in $K$,   
let us consider a new concept name $\overline{C}$, (representing the negation of $C$) such that $C \sqcap \overline{C} \sqsubseteq \bot$. Let $\lingconc_K$ be the set of all such  $C$ and $\overline{C}$, and let ${\cal T}_{Constr}$ the set of all subsumptions $C \sqcap \overline{C} \sqsubseteq \bot$.
A set $\{ D_1,\ldots,D_m\}$ of concepts in $\lingconc_K$ is {\em consistent with $K$} if ${\cal T}_{Strict} \cup {\cal T}_{Constr} \not \models_{\elpb}  D_1 \sqcap  \dots \sqcap D_m \sqsubseteq \bot$. 

\begin{definition}\label{def-canonical-model-DL} 
Given a ranked knowledge base $K=\langle  {\cal T}_{strict}, {\cal T}_{C_1}, \ldots, {\cal T}_{C_k}, {\cal A}  \rangle$
an $\elpb$ interpretation $I=\langle \Delta, \cdot^I \rangle$ 
is {\em canonical}  for $K$ if  $I$ satisfies $ {\cal T}_{strict}$ and, 
for any set of concepts $\{ D_1,\ldots,D_m\} \subseteq \lingconc_K$ consistent with $K$, 
there exists  
a domain element $x \in \Delta$ such that,
for all $i=1, \ldots, m$, $x \in C^I$, if $D_i=C$, and $x \not \in C^I$, if $D_i=\overline{C}$.
\end{definition}
The idea is that, in a canonical model for $K$, any conjunction of concepts occurring in $K$, or their complements, when consistent with $K$, must have an instance in the domain. 
Existence of  canonical interpretations is guaranteed for knowledge bases 
which are consistent under the preferential (or ranked) semantics for typicality. 
$\elpb$ with typicality is indeed a fragment of the description logic $\shiq$ with typicality, for which existence of canonical models of consistent knowledge bases was proved \cite{GiordanoFI2018}.

In agreement with the preferential interpretations of typicality logics,
we further require that, if there is some $C_h$-element in a model, then there is at least one $C_h$-element satisfying all typicality inclusions for $C_h$ (i.e., a prototypical $C_h$-element).

\begin{definition} 
An $\elpb$ interpretation $I=\langle \Delta, \cdot^I \rangle$  is {\em $\tip$-compliant for $K$} if, I satisfies  ${\cal T}_{Strict}$ and, for all $C_h \in {\cal C}$  such that $C_h^I \neq \emptyset$, there is some $x \in C_h^I$ such that $x$ satisfies all  defeasible inclusions in ${\cal T}_{C_h}$. 
\end{definition}
In a canonical and  $\tip$-compliant  interpretation for $K$, for each $C_j \in {\cal C}$, the relation $\leq_{C_j}$  on the domain $\Delta$  provides a preferential interpretation for the typicality concept $\tip(C_j)$ 
as $min_{<_{C_j}}(C_j^I)$, in which all typical $C_j$ satisfy all typicality inclusions in ${\cal T}_{C_h}$.

Existence of a $\tip$-compliant canonical interpretation is not guaranteed for an arbitrary knowledge base.  
For instance,  a knowledge base whose typicality inclusions conflict with strict ones (e.g, $\tip(C_j) \sqsubseteq D$ and   $C_j \sqcap D \sqsubseteq \bot$) has  no $\tip$-compliant  interpretation. 
However, existence of  $\tip$-compliant interpretations is guaranteed for knowledge bases 
which are consistent under the preferential (or ranked) semantics for typicality  
(see Appendix A, Proposition \ref{prop:existence-Tcompliant}), 
 and consistency can be tested in polynomial time in Datalog \cite{SROEL_FI_2018}.

For a ranked knowledge base
$K=\langle  {\cal T}_{strict}, {\cal T}_{C_1}, \ldots, {\cal T}_{C_k}, {\cal A}  \rangle$, and a given $\elpb$ interpretation $I=\langle \Delta, \cdot^I \rangle$, 
 the  strict modular partial order relations $<_{C_1}, \ldots, <_{C_k}$ over $\Delta$, defined according to Definition \ref{total_preorder} above,
determine the relative typicality of domain elements w.r.t. each concept $C_j$. 
Clearly, the different preference relations $<_{C_j}$ do not need to agree, as seen in the introduction.

\section{Combining multiple preferences into a global preference}

We are interested in defining  a notion of typical $C$-element, and defining an interpretation of $\tip(C)$,
which works for all concepts $C$, not only for the distinguished concepts in ${\cal C}$.
This can be used to evaluate subsumptions of the form $\tip(C) \sqsubseteq D$ when $C$ does not belong to ${\cal C}$.
We address this problem by introducing a notion of {\em multipreference concept-wise} interpretation, 
which generalizes the notion of preferential interpretation \cite{KrausLehmannMagidor:90} by allowing multiple preference relations and, then, combining them in a single (global) preference. 
Let us consider the following example:

\begin{example} \label{exa:student}
Let $K$ be the ranked KB $\langle {\cal T}_{strict},  {\cal T}_{Employee}, {\cal T}_{Student}, {\cal T}_{PhDStudent}, {\cal A} \rangle$ (with empty ${\cal A} =\emptyset$), 
containting the strict inclusions:

$\mathit{Employee  \sqsubseteq  Adult}$ \ \ \ \ \ \ \ \ \ \ \ \ \ \ \   $\mathit{Adult  \sqsubseteq  \exists has\_SSN. \top}$  
 \ \ \ \ \ \ \ \ \ \ \   $\mathit{PhdStudent  \sqsubseteq  Student}$ 
 
$\mathit{Young  \sqcap NotYoung \sqsubseteq  \bot}$ \ \ \ \ \ \ \ \
$\mathit{\exists hasScholarship.\top  \sqcap Has\_no\_Scholarship \sqsubseteq  \bot}$

\noindent
The ranked TBox ${\cal T}_{Employee} =\{(d_1,0), (d_2,0)\}$ contains the defeasible inclusions:

$(d_1)$ $\mathit{\tip(Employee) \sqsubseteq NotYoung}$  \ \ \ \ \ \ \ \ \  \ \ \ \ \  $(d_2)$ $\mathit{\tip(Employee) \sqsubseteq \exists has\_boss.Employee}$

\noindent
the ranked TBox ${\cal T}_{Student}= \{(d_3,0),(d_4,1), (d_5,1)\}$ contains the defeasible inclusions:

$(d_3)$ $\mathit{\tip(Student) \sqsubseteq  \exists has\_classes.\top}$   \ \ \ \ \ \ \ \ \ \  $(d_4)$ $\mathit{\tip(Student) \sqsubseteq Young}$

$(d_5)$ $\mathit{\tip(Student) \sqsubseteq  Has\_no\_Scholarship}$

\noindent
and the ranked TBox ${\cal T}_{PhDStudent}=\{  (d_6, 0), (d_7,1)\}$ contains the inclusions:

$(d_6)$ $\mathit{\tip(PhDStudent) \sqsubseteq  \exists hasScholarship.Amount}$ \ \ \ \ \ \ \ \  $(d_7)$ $\mathit{\tip(PhDStudent) \sqsubseteq Bright}$

\end{example}
We might be interested to check whether typical Italian students are young or whether typical employed students are young.
This would require the typicality inclusions \linebreak $\mathit{\tip(Student \sqcap Italian) \sqsubseteq  Young}$ and $\mathit{\tip(Employee \sqcap Student) \sqsubseteq Young}$ to be evaluated. 
\linebreak
Nothing should prevent Italian students from being young 
(irrelevance). 
Also, we expect not to conclude 
 that typical employed students are young nor that they are not, as typical students and typical employees have conflicting properties concerning age. However, we would like to conclude that  typical employed students have a boss, have classes and have no scholarship, as they should inherit the properties of typical students and of typical employees which are not overridden 
 (i.e., there is no blocking of inheritance). 
 As PhD students are students, they should inherit  all the typical properties of Students, except having no scholarship, which is overridden by ($d_6$).

To evaluate conditionals $\tip(C) \sqsubseteq D$ for any concept $C$ we introduce a concept-wise multipreference interpretation, that combines the preference relations $<_{C_1}, \ldots, <_{C_k}$ into a single {\em (global)} preference relation $<$ and interpreting   $\mathit{\tip(C)}$ as $\mathit{(\tip(C))^I}=$ $ min_{<}(C^I)$.
The relation $<$ should be defined starting from the preference relations $<_{C_1}, \ldots, <_{C_k}$ also considering specificity.

Let us consider the simplest notion of {\em specificity} among concepts, based on the subsumption hierarchy  
(one of the notions considered  for ${\cal DL}^N$  \cite{bonattiAIJ15}).

\begin{definition}[Specificity] \label{specificity}
Given a ranked $\elpb$ knowledge base
$K=\langle  {\cal T}_{strict},$ $ {\cal T}_{C_1}, \ldots,$ $ {\cal T}_{C_k}, {\cal A}  \rangle$ over the set of concepts ${\cal C}$, 
and given  two concepts  $C_h, C_j \in {\cal C}$,
 $C_h$ is {\em more specific than}  $C_j$ (written $C_h \succ C_j$) 
  if $ {\cal T}_{strict} \models_{\elpb} C_h \sqsubseteq C_j$  and $ {\cal T}_{strict} \not\models_{\elpb} C_j \sqsubseteq C_h$. 
\end{definition}
Relation $\succ$ is irreflexive and transitive \cite{bonattiAIJ15}.  Alternative notions of specificity can be used, based, for instance, on the rational closure ranking.

We are  ready to define a notion of multipreference interpretation. 
Let a relation $<_{C_i}$ be {\em well-founded} when there is no infinitely-descending chain of domain elements $x_1<_{C_i} x_0, \; x_2<_{C_i} x_1, \; x_3<_{C_i} x_2, \ldots $.
 \begin{definition}[concept-wise multipreference interpretation]\label{def-multipreference-int}  
A  (finite) concept-wise multipreference interpretation (or cw$^m$-interpretation) is a tuple $\emme= \langle \Delta, <_{C_1}, \ldots,<_{C_k}, <, \cdot^I \rangle$
such that: (a)  $\Delta$ is a non-empty domain;   
\begin{itemize} 

\item[(b)] for each $i=1,\ldots, k$, $<_{C_i}$ is an irreflexive, transitive, well-founded and modular relation over $\Delta$;

\item[(c)]  the (global) preference relation $<$ is defined from  $<_{C_1}, \ldots,<_{C_k}$ as follows: \ \ 
\begin{align*}
x <y  \mbox{ iff \ \ } 
(i) &\  x <_{C_i} y, \mbox{ for some } C_i \in {\cal C}, \mbox{ and } \\
(ii) & \ \mbox{  for all } C_j\in {\cal C}, \;  x \leq_{C_j} y \mbox{ or }  \exists C_h (C_h \succ C_j  \mbox{ and } x <_{C_h} y )
\end{align*}
\item[(d)]  $\cdot^I$ is an interpretation function, as defined in $\elpb$ interpretations 
(see Section \ref{sec:ALC}),
with the addition that, for typicality concepts, we let: 
$(\tip(C))^I = min_{<}(C^I)$,
where $Min_<(S)= \{u: u \in S$ and $\nexists z \in S$ s.t. $z < u \}$.

\end{itemize}
\end{definition}
Notice that the
 relation $<$ is defined from $<_{C_1}, \ldots,<_{C_k}$  based on a {\em modified} Pareto condition:
$x< y$ holds if there is at least a $C_i \in {\cal C}$ such that $ x <_{C_i} y$ and,
 for all $C_j \in {\cal C}$,   either $x \leq_{C_j} y$ holds or, in case it does not, there is some $C_h$ more specific than $C_j$ such that $x <_{C_h} y$ (preference  $<_{C_h}$ in this case overrides $<_{C_j}$).
 For instance, in Example \ref{exa:student}, for two domain elements $x, y$, both instances of $\mathit{PhDStudent, Student,}$ $\mathit{ \exists has\_Classes.\top,Young}$, and such that $x$ is instance of  $\mathit{has\_no\_Scholarship}$, while $\mathit{y}$ is not, we have that $\mathit{x<_{Student} y}$ and $\mathit{y<_{PhDStudent} x}$. As $\mathit{PhDStudent}$ is more specific than $\mathit{Student}$, globally we get $\mathit{y< x}$.
We can prove  
the following result. 
\begin{proposition} \label{properties_global_pref}
Given a cw$^m$-interpretation $\emme= \langle \Delta, <_{C_1}, \ldots,<_{C_k}, <, \cdot^I \rangle$,
 relation $<$ is an irreflexive, transitive and well-founded relation.
\begin{proof} 
Well-foundedness of $<$ is immediate from the restriction to finite models.

To prove irreflexivity and transitivity of $<$, we exploit the fact that each $<_{C_i}$ is assumed to be an irreflexive, transitive, well-founded and modular relation on $\Delta$ (see Definition \ref{def-multipreference-int}).
Irreflexivity of $<$ follows easily from the irreflexivity of the  $<_{C_h}$'s as, for $x<x$ to hold, $x <_{C_h} x$ should hold for some $C_h$, which is not possible as $<_{C_h}$ is irreflexive. 

To prove transitivity of $<$, we prove transitivity of $\leq$ defined as follows:
\begin{align*}
x \leq y  \mbox{ iff } \mbox{ for all }  C_j \in {\cal C} \ \ 
(i) &\  x \leq_{C_j} y,  \mbox{ or }\\
(ii) & \mbox{ exists } C_h\in {\cal C}, \; (C_h \succ C_j  \mbox{ and } x <_{C_h} y )
\end{align*}
It is easy to see that the global preference relation $<$ introduced in point (c) of Definition \ref{def-multipreference-int} can be equivalently defined as: 
$x <y  \mbox{ iff \ } (x \leq y$ and not $y \leq x $). Transitivity of $<$ follows from transitivity of $\leq$.

To prove transitivity of $\leq$, let us assume that $x\leq y$ and $y \leq z$ hold. We prove that $x\leq z$ holds by proving that:
for all $C_j \in {\cal C}$,  $x  \leq_{C_j} z$ holds (call this case $(i)^{x,z}_{C_j}$) or  there is a $C_h$ such that  $C_h \succ C_j$  and $x <_{C_k} z$ (call this case $(ii)^{x,z}_{C_j}$)).

As $x \leq y$ holds, for all $C_j \in {\cal C}$,  $x  \leq_{C_j} y$ (case $(i)_1$) or 
there is a $C_h$ such that  $C_h \succ C_j$ and $x <_{C_k} y$ (case $(ii)_1$).
Similarly, as $y \leq z$ holds,  for all $C_j \in {\cal C}$, $y \leq_{C_j} z$ (case $(i)_2$) or  there is a $C_r$ such that  $C_r \succ C_j$ and $x <_{C_r} y$ (case $(ii)_2$).
Let us consider the different possible combination of cases in which $x \leq y$ and $y \leq z$ hold, for each $C_j$:

Case $(i)_1$-$(i)_2$: In this case, $x  \leq_{C_j} y$ and $y \leq_{C_j} z$ hold. By transitivity of $\leq_{C_j}$,  $x  \leq_{C_j} z$ (i.e., condition $(i)^{x,z}_{C_j}$ is satisfied).

Case $(ii)_1$-$(i)_2$:  In this case, $y \leq_{C_j} z$ ,  
and  there is a $C_{h}$ such that  $C_{h} \succ C_j$ and $x <_{C_{h}} y$. Let $C_{h}$ be maximally specific among all concepts $C \in {\cal C}$ such  that $C \succ C_j$ and $x <_{C} y$.

If $y \leq_{C_h} z$ is the case, from $x <_{C_h} y$, we get $x <_{C_h} z$, so that: 
there is a $C_h$ such that  $C_h \succ C_j$ and $x <_{C_h} z$, i.e., condition $(ii)^{x,z}_{C_j}$ is satisfied. 
Otherwise, if $z <_{C_h} y$,  as $y \leq z$, there must be a $C_{r}$ such that  $C_{r} \succ C_{h}$ and $y <_{C_{r}} z$.
If $x \leq_{C_{r}} y$, we can conclude that $x <_{C_{r}} z$. From $C_{r} \succ C_{h} \succ C_j$, by transitivity,  $C_{r} \succ C_j$,
i.e. condition $(ii)^{x,z}_{C_j}$ is satisfied.
If $x \leq_{C_{r}} y$ does not hold, i.e. $y <_{C_{r}} x$, as $x \leq y$, there must be a $C_w \in {\cal C}$ such that $C_w \succ C_r$ and $x <_{C_w} y$.
However, this is not possible, as it would be  $C_w \succ C_r \succ C_{h} \succ C_j$ and
 we have chosen $C_h$ to be maximally specific among the concepts $C \in {\cal C}$ such  that $C \succ C_j$ and $x <_{C} y$, a contradiction.
  
The cases $(i)_1$-$(ii)_2$ and $(ii)_1$-$(ii)_2$  can be proved in a similar way.
\end{proof}
\end{proposition}
In a cw$^m$-interpretation we have assumed each $<_{C_j}$ to be {\em any} irreflexive, transitive, modular and well-founded relation.
In a cw$^m$-model of $K$, the preference relations $<_{C_j}$'s will be defined from the  ranked TBoxes ${\cal T}_{C_j}$'s according to Definition \ref{total_preorder}.
\begin{definition}[cw$^m$-model of $K$]\label{cwm-model} 
Let 
$K=\langle  {\cal T}_{strict},$ $ {\cal T}_{C_1}, \ldots,$ $ {\cal T}_{C_k}, {\cal A}  \rangle$ be a ranked $\elpb$ knowledge base over  ${\cal C}$   
and 
$I=\langle \Delta, \cdot^I \rangle$ an $\elpb$ 
interpretation for $K$.
A {\em concept-wise multipreference model} (or {\em cw$^m$-model})  of $K$ is  a cw$^m$-interpretation ${\emme}=\langle \Delta,<_{C_1}, \ldots, <_{C_k}, <, \cdot^I \rangle$ 
such that: 
for all $j= 1, \ldots, k$,  $<_{C_j}$ is 
defined from  ${\cal T}_{C_j}$ and $\cdot^I$, according to Definition \ref{total_preorder};
$\emme$  satisfies  all strict inclusions inclusions in $ {\cal T}_{strict}$ 
and all assertions in ${\cal A}$.
\end{definition}
As the preferences $<_{C_j}$'s,  defined according to Definition \ref{total_preorder}, 
are irreflexive, transitive, well-founded and modular relations over $\Delta$, a cw$^m$-model $\emme$ is indeed a cw$^m$-interpretation. By definition $\emme$ satisfies all strict inclusions  
and assertions in $K$, but is not required to satisfy all typicality inclusions $\tip(C_j) \sqsubseteq  D$ in $K$,
unlike in preferential typicality logics \cite{lpar2007}. 

Consider, in fact, a situation in which typical birds are fliers and typical fliers are birds ($\tip(B) \sqsubseteq F$ and $\tip(F) \sqsubseteq B$).  In a  cw$^m$-model two domain elements $x$ and $y$, which are both birds and fliers, might be incomparable wrt $<$, as $x$ is more typical than $y$ as a bird, while $y$ is more typical than $x$ as a flier, even if one of them is minimal wrt $<_{Bird}$ and the other is not. In this case, they will be both minimal wrt $<$. In preferential logics, we would conclude that $\tip(B) \equiv \tip(F)$, which is not the case under the cw$^m$-semantics. This implies that the notion of cw$^m$-entailment (defined below) is not stronger  
than preferential entailment. 
It is also not weaker as, 
for instance, in Example \ref{exa:student}, cw$^m$-entailment allows to conclude that typical employed students have a boss, have classes and no scholarship (although defaults $(d_1)$ and $(d_4)$ are conflicting), while
neither preferential entailment nor the rational closure would allow such conclusions; cw$^m$-entailment does not suffer from inheritance blocking, 
and is then incomparable with preferential entailment and with entailment under rational closure, being neither weaker nor stronger.

The notion of 
cw$^m$-entailment  
exploits canonical and $\tip$-compliant cw$^m$-models of $K$.
A cw$^m$-model  $\emme= \langle \Delta, <_{C_1}, \ldots,<_{C_k}, <, \cdot^I \rangle$ is {\em canonical ($\tip$-compliant)} for $K$ if the $\elpb$ interpretation  $\langle \Delta, \cdot^I \rangle$ is canonical ($\tip$-compliant) for $K$.
\begin{definition}[cw$^m$-entailment] \label{cwm-entailment}
An inclusion $\tip(C) \sqsubseteq C_j$ is cw$^m$-entailed  
from $K$ (written $K \models_{cw^m} \tip(C) \sqsubseteq C_j$) if
$\tip(C) \sqsubseteq C_j$ is satisfied in all canonical and $\tip$-compliant cw$^m$-models  
$\emme$ of $K$.
\end{definition}
It can be proved that cw$^m$-entailment satisfies the KLM postulates of a preferential consequence relation (see Appendix A, Proposition \ref{prop:KLM_properties}).

\section{Reasoning under the cw-multipreference semantics}
   
In this section we consider the problem of checking cw$^m$-entailment of a typicality subsumption  $\tip(C) \sqsubseteq D$ as a problem of determining preferred answer sets. Based on this formulation, that we prove to be sound and complete, we show that the problem is in $\Pi^p_2$. We exploit \emph{asprin} \cite{BrewkaAAAI15} to compute preferred answer sets. The proofs for this section can be found in Appendix C.

In principle, for checking $\tip(C) \sqsubseteq D$ 
we would need to consider all possible typical $C$-elements in all possible canonical and $\tip$-compliant cw$^m$-model of $K$, and verify whether they are all instances of $D$.
However, we will prove that  
it is sufficient to consider,
among all the (finite) cw$^m$-models of $K$, the polynomial $\elpb$ models that we can construct using the $\elpb$ fragment of the materialization calculus for $\sroel$  \cite{KrotzschJelia2010}, by considering all alternative interpretations for a distinguished element $aux_C$, representing a prototypical $C$-element. In preferred models, which minimize the violation of typicality inclusions by $aux_C$, it indeed represents a typical $C$-element. An interesting result is that neither we need to consider all the possible interpretations for constants in the model nor to minimize violation of typicalities for them. Essentially, when evaluating the properties of typical employed students we are not concerned with the typicality (or atypicality) of other constants in the model (e.g., with typical cars, with typical birds, and with typical named individuals). Unlike a previous semantics by Giordano and Theseider Dupr{\'e} \shortcite{TPLP2016}, which generalizes rational closure by allowing typicality concepts on the rhs of inclusions, 
we are not required to consider all possible alternative interpretations and ranks of individuals in the model. We will see, however, that we do not loose solutions (models) in this way.

In the following we first describe how answer sets of a base program, corresponding to cw$^m$-models of $K$, are generated. Then we describe
how preferred models can be selected, where $\mathit{aux_C}$ represent a typical $C$-element.

We will assume that assertions ($C(a)$ and $r(a,b)$) are represented using nominals as inclusions
(resp., $\{a\} \sqsubseteq A$ and $\{a\} \sqsubseteq \exists R. \{b\}$), where a nominal $\{a\}$ is a concept containing a single element and $(\{a\})^I=\{a^I\}$.
We also assume that the knowledge base $K$ is in {\em normal form} \cite{rifel}, where a typicality inclusion  $\tip(B) \sqsubseteq C$ is in normal form when $B, C \in N_C$  \cite{TPLP2016}. Extending the results in \cite{KrotzschJelia2010}, it can be proved that, given a KB, a semantically equivalent KB in normal form (over an extended signature) can be computed in linear time. We refer to a previous paper \cite{SROEL_FI_2018} for the details on normalization.

The base program $\Pi(K,C,D)$ for the (normalized) knowledge base $K$ and typicality subsumption 
$\tip(C) \sqsubseteq D$ is composed of three parts, 
$\Pi(K,C,D)= \Pi_{K} \cup \Pi_{IR} \cup \Pi_{C,D}$.

$\Pi_{K}$ is the representation of $K$ in Datalog \cite{KrotzschJelia2010}, where, to keep a DL-like notation, we do not follow the convention where variable names start with uppercase;
in particular, $A$, $C$, $D$ and $R$, are intended as ASP constants corresponding to the same class/role names in $K$. In this representation,
$\mathit{nom(a)}$, $\mathit{cls(A)}$, $\mathit{rol(R)}$ are used for
$\mathit{a \in N_I}$ , $\mathit{A \in N_C}$, $\mathit{R \in N_R}$, and, for example,
	$\mathit{subClass(a,C)}$, $\mathit{subClass(A,C)}$ are used for $C(a)$,  $A \sqsubseteq C$.
Additionally, $\mathit{subTyp(C,D,N)} $ is used for $ \mathit{T(C) \sqsubseteq D} $ having rank $N$,
and the following
definitions for distinguished concepts, typical properties, and valid ranks, will be used in defining preferences:
\begin{tabbing}
$ \mathit{dcls(C) \leftarrow  subTyp(C,D,N)} $\\
$ \mathit{tprop(C,D) \leftarrow subTyp(C,D,N)} $ \\
$ \mathit{validrank(C,N) \leftarrow subTyp(C,D,N)} $
\end{tabbing}
For each distinguished concept $C_i$, $ \mathit{auxtc(aux\_Ci,Ci)}$ is included, where 
$\mathit{aux\_Ci}$ is an auxiliary individual name.  Other auxiliary constants (one for each inclusion $A \sqsubseteq \exists R.B$) are needed \cite{KrotzschJelia2010} to deal with existential rules.

$\Pi_{IR}$ contains the subset  of the
inference rules (1-29) for instance checking \cite{KrotzschJelia2010} that is relevant for $\elpb$ (reported in Appendix B),
for example	$\mathit{inst(x,z) \leftarrow}$ $\mathit{ subClass(y,z),inst(x,y)} $; for
$\bot$, an additional rule is used: $ \mathit{\leftarrow bot(z), inst(x,z)}$. 
Additionally, $\Pi_{IR}$ contains the version of the same rules for subclass checking 
(where $\mathit{inst\_sc(A,B,A)}$ represents $A  \sqsubseteq B$ \cite{KrotzschJelia2010}),
and then the following rule encodes specificity  $C_h \succ C_j$:
\begin{tabbing}
$\mathit{morespec(Ch,Cj) \leftarrow dcls(Ch),dcls(Cj),inst\_sc(Ch,Cj,Ch), not  ~ inst\_sc(Cj,Ch,Cj)}$ 
\end{tabbing}

$\Pi_{IR}$ also contains the following rules:

\begin{tabbing}

(a)	$ \mathit{
\{inst(aux_C,D)\}\ \leftarrow dcls(Ci),inst(aux_C,Ci),tprop(Ci,D)
	} $\\

(b)	$ \mathit{
inst(Y,Ci) \leftarrow auxtc(Y,Ci), inst(X,Ci)
	} $\\
	
(c)	$ \mathit{
typ(Y,Ci) \leftarrow auxtc(Y,Ci),inst(Y,Ci)
	} $\\

(d)	$ \mathit{
inst(Y,D) \leftarrow subTyp(Ci,D,N),typ(Y,Ci)
	} $

\end{tabbing}
Rule (a)  generates alternative answer sets (corresponding to different interpretations) where $\mathit{aux_C}$ may have the typical
properties of the concepts it belongs.
The constant $aux\_Ci$, such that $auxtc(aux\_Ci,Ci)$ holds, represents a typical $Ci$ (a minimal element wrt. $\leq_{C_i}$) only in case it is an instance of $C_i$ (i.e., $\mathit{inst(aux\_Ci, Ci)}$ holds).
Rule (b) establishes that, if there is an instance $x$ of concept $C_i$ in the interpretation, then  $aux\_Ci$ must be an instance of $C_i$  (it models $\tip$-compliance) and, 
by rule (c), $aux\_Ci$ is a typical instance of $C_i$, i.e.,
it is minimal wrt. $\leq_{C_i}$ among $C_i$-elements in the interpretation at hand. 
By rule (d), a typical instance of $C_i$ has all typical properties of $C_i$.
The rules (b)-(d) only allow to derive conclusions involving $aux\_Ci$ constants.

$\Pi_{C,D}$ contains (if necessary) normalized axioms defining $C,D$ in $\tip(C) \sqsubseteq D$ in terms of other concepts 
(e.g., replacing $\mathit{\tip(Employee \sqcap Student) \sqsubseteq Young}$ with $\mathit{\tip(A) \sqsubseteq Young}$ and $\mathit{A \sqsubseteq Employee}$,  $\mathit{A \sqsubseteq Student}$ and $\mathit{Employee \sqcap Student \sqsubseteq A}$)
plus the facts
$ \mathit{auxtc(aux_C,C)}$, $ \mathit{nom(aux_C)}$, $ \mathit{inst(aux_C,C).}$

\begin{proposition} \label{AS to models}
Given a normalized ranked knowledge base $K=\langle  {\cal T}_{strict},$ $ {\cal T}_{C_1}, \ldots,$ $ {\cal T}_{C_k}, {\cal A}  \rangle$ over the set of concepts ${\cal C}$,  
and a (normalized) subsumption  $C \sqsubseteq D$: 
\begin{itemize}
\item[(1)]
if there is an answer set $S$ of the ASP program $\Pi(K,C,D)$,  
such that $inst(aux_C, D) \not \in S$,
then there is a $\tip$-compliant cw$^m$-model $\emme=\langle \Delta, <_{C_1}, \ldots, <_{C_k}, <, \cdot^I \rangle $ for $K$  
that falsifies the subsumption  $C \sqsubseteq D$.

\item[(2)]
if there is a $\tip$-compliant cw$^m$-model $\emme=\langle \Delta, <_{C_1}, \ldots, <_{C_k}, <, \cdot^I \rangle $ of $K$  
that falsifies the subsumption  $C \sqsubseteq D$,
then there is an answer set $S$ 
of $\Pi(K,C,D)$, 
such that $\mathit{ inst(aux_C, D) \not \in S}$.
\end{itemize}
\end{proposition}
We exploit the idea of identifying the minimal $C$-elements in a canonical cw$^m$-model of $K$, as the $aux_C$ elements of  the {\em preferred} answer sets of $\Pi(K,C,D)$.
\begin{definition}
Let $S$ and $S'$ be answer sets of $\Pi(K,C,D)$.
{\em $S'$ is preferred to $S$} if $aux_C$ in $S'$ (denoted as $aux_C^{S'}$) is globally preferred to  $aux_C$ in $S$ (denoted as $aux_C^S$), 
that is, $aux_C^{S'} < aux_C^S$, 
defined according to Definition \ref{def-multipreference-int}, point (c), provided that
 relations  $aux_C^{S'} \leq_{C_j} aux_C^S$  are defined according to Definition \ref{total_preorder}, by letting:
\begin{quote}
$\mathit{{\cal T}^l_{C_i}(aux_C^{S})} =\{ \mathit{B \mid \;  inst(aux_C, C_i) \not\in S}$ 
or $\mathit{inst(aux_C, B)} \in S$, 
for $\tip(C_i) \sqsubseteq B \in K $\},  
	
\end{quote}
i.e., $\mathit{{\cal T}^l_{C_i}(aux_C^{S})}$ contains the $B$'s such 
that $C_i \sqsubseteq B$ is satisfied in $S$
for some typicality inclusion $\tip(C_i) \sqsubseteq B$ in $K$; 
and similarly for $S'$.
The strict relation $aux_C^{S'} <_{C_j} aux_C^S$ is defined accordingly.
\end{definition}
Essentially, we compare $S$ and $S'$ identifying the concepts of which $aux_C$ is an instance in $S$ and in $S'$ and evaluating which defaults are satisfied for $aux_C$ in $S$ and in $S'$, using the same criteria used for comparing domain elements in Section  \ref{sez:RC}.

The selection of preferred answer sets, the ones where $aux_C$ is in $min_{<}(C^I)$, and then in $(\tip(C))^I$, can be done in \emph{asprin} with the following preference specification: 
\begin{tabbing}	
	$ \mathit{\#preference(p,multipref)\{ dcls(Ci) : dcls(Ci) ; morespec(Ci,Cj) : dcls(Ci),dcls(Cj) ; } $ \\
	~ ~ ~ ~ ~ ~ ~ ~ $ \mathit{ inst(auxC,E) : tprop(Ci,E), dcls(Ci) ; subTyp(Ci,E,R) : subTyp(Ci,E,R) ; } $ \\
	~ ~ ~ ~ ~ ~ ~ ~ $ \mathit{ validrank(Ci,R) : validrank(Ci,R) \} } $ \\
	$ \mathit{\#optimize(p)}$	
\end{tabbing}
requiring optimization wrt $p$ which is a preference of type $\mathit{multipref}$, a preference type defined by the preference program below (exploiting the fact that \emph{asprin}, among other things, generates from the specification a fact $ \mathit{preference(p,multipref)}$).

In \emph{asprin} preference programs,
defining whether an answer set $S$ is preferred to $S'$ according to a preference $P$ amounts
to defining a predicate $better(P)$ for the case where $P$ is of the type being defined; 
the predicates $holds$ and $holds'$ are used to check whether the atoms in the preference specification hold in $S$ and $S'$, respectively. 
In the following,
$\mathit{better(p), bettereqwrt(Ci), betterwrt(Ci)}$, correspond to $<$, $\leq_{C_i}$, $<_{C_i}$, respectively, for $aux_C^{S}$ and $aux_C^{S'}$, comparing what $holds$ for $aux_C$ to what $holds'$ for it; $\mathit{moreprop}$ and $\mathit{samenumprop}$ verify whether more (or the same number of ) typicality inclusions
of rank $R$ are satisfied by $aux_C$ in $S$ wrt $S'$:

\begin{tabbing}
	$ \mathit{\#program ~ preference(multipref)}$
	\\
	
	$ \mathit{
		better(P) \leftarrow
		preference(P,multipref),
		holds(dcls(Ci)), } $ \\
	~ ~ ~ ~ ~ ~ ~ ~ $ \mathit{ betterwrt(Ci),
		noattack(Cj) : holds(dcls(Cj)) } $
	\\ 
	
	$ \mathit{
		noattack(Cj) \leftarrow
		holds(dcls(Cj)),
		bettereqwrt(Cj) } $
	\\
	
	$ \mathit{
		noattack(Cj) \leftarrow
		holds(dcls(Cj)),
		holds(dcls(Ch)),
		holds(morespec(Ch,Cj)),
		betterwrt(Ch)} $
	\\
	
	$ \mathit{bettereqwrt(Ci) \leftarrow betterwrt(Ci)} $
	\\
	
	$ \mathit{
		bettereqwrt(Ci) \leftarrow
		holds(dcls(Ci)),
		samenumprop(Ci,R) : holds(validrank(Ci,R))} $
	\\
	
	$ \mathit{
		betterwrt(Ci) \leftarrow
		holds(dcls(Ci)),
		moreprop(Ci,R),} $ \\
	~ ~ ~ ~ ~ ~ ~ ~ $ \mathit{
		samenumprop(Ci,R1) : holds(validrank(Ci,R1)),R1>R} $
	\\	
	
	$ \mathit{	
		moreprop(Ci,R) \leftarrow
		holds(validrank(Ci,R)), }$ \\
	~ ~ ~ ~ ~ ~ ~ ~ $ \mathit{
		\#sum \{ -1,E : sat(auxC,Ci,E), holds(subTyp(Ci,E,R)); } $ \\
	~ ~ ~ ~ ~ ~ ~ ~ ~ ~ ~ ~ ~ $ \mathit{
		1,E : sat1(auxC,Ci,E), holds(subTyp(Ci,E,R)) \} -1} $
	\\
	
	$ \mathit{
		sat(auxC,Ci,E) \leftarrow holds(X), X=inst(auxC,E),holds(subTyp(Ci,E,R))} $
	\\
	
	$ \mathit{
		sat(auxC,Ci,E) \leftarrow not ~ holds(X), X=inst(auxC,Ci),holds(subTyp(Ci,E,R))} $
	\\
	
	$ \mathit{	
		sat1(auxC,Ci,E) \leftarrow holds'(X), X=inst(auxC,E),holds(subTyp(Ci,E,R))} $
	\\
	
	$ \mathit{		
		sat1(auxC,Ci,E) \leftarrow not ~ holds'(X), X=inst(auxC,Ci),holds(subTyp(Ci,E,R))} $
	\\
	
	$ \mathit{		
		samenumprop(Ci,R) \leftarrow
		holds(validrank(Ci,R)), }$ \\
	~ ~ ~ ~ ~ ~ ~ ~ $ \mathit{
		0 \#sum \{ -1,E : sat(auxC,Ci,E), holds(subTyp(Ci,E,R)); }$ \\
	~ ~ ~ ~ ~ ~ ~ ~ ~ ~ ~ ~ ~  ~ ~ $ \mathit{
		1,E : sat1(auxC,Ci,E), holds(subTyp(Ci,E,R)) \} 0 } $

\end{tabbing}
Let us call  $\mathit{Pref}$ the preference specification and the preference program defined above; checking whether $\tip(C) \sqsubseteq D$ is cw$^m$-entailed amounts to checking whether $inst(aux_C, D)$ is in all preferred answer sets of $\Pi(K,C,D)$ according to $\mathit{Pref}$.
 
\begin{proposition} \label{AS to pref-models}
Given a normalized ranked knowledge base $K=\langle  {\cal T}_{strict},$ $ {\cal T}_{C_1}, \ldots,$ $ {\cal T}_{C_k}, {\cal A}  \rangle$ over the set of concepts ${\cal C}$,  
and a subsumption  $\tip(C) \sqsubseteq D$, we can prove the following:
\begin{itemize}
\item[(1)]
if there is a canonical and $\tip$-compliant cw$^m$-model $\emme=(\Delta, <_{C_1}, \ldots, <_{C_k}, <, \cdot^I)$ of $K$  
that falsifies  $\tip(C) \sqsubseteq D$, then there is a preferred answer set $S$ 
of $\Pi(K,C,D)$ according to  $\mathit{Pref}$,
such that $inst(aux_C, D) \not \in S$.

\item[(2)]
if there is a preferred answer set $S$ of $\Pi(K,C,D)$  according to  $\mathit{Pref}$,  such that $inst(aux_C, D) \not \in S$,
then there is a canonical and $\tip$-compliant  cw$^m$-model $\emme=(\Delta, <_{C_1}, \ldots, <_{C_k}, <, \cdot^I)$ of $K$  
that falsifies  $\tip(C) \sqsubseteq D$.
\end{itemize}
\end{proposition}
Propositions  \ref{AS to models} and \ref{AS to pref-models} tell us that, for computing cw$^m$-entailment,  it is sufficient to consider the polynomial $\tip$-compliant cw$^m$-models of $K$ corresponding to answer sets  $S$ of $\Pi(K,$ $C,D)$\footnote{Note that, for verifying cw$^m$-entailment of $\tip(C) \sqsubseteq D$,  all answer sets of $\Pi(K,$ $C,D)$ have to be considered and checking whether  $\Pi(K,C,D) \cup \{- inst(aux_C, D)\} $ has no preferred answer sets would not be correct.}. 
A $\Pi^p_2$ upper bound on the complexity of cw$^m$-entailment can be proved based on the the above formulation of cw$^m$-entailment as a problem of computing preferred answer sets. The $\Pi^p_2$-hardness can be proved by providing a reduction of the minimal entailment problem of {\em positive disjunctive logic programs}, which was proved to be a
 \textsc{$\Pi^P_2$}-hard problem by Eiter and Gottlob \shortcite{Eiter95}. 
\begin{proposition} \label{prop: upper bound}
Deciding cw$^m$-entailment is a  $\Pi^p_2$-complete problem.
\end{proposition}

\subsection{Some experimental results}

For Example \ref{exa:student}, we actually get that typical employed students have a boss, but not that they are young: there are, in fact, two preferred answer sets, with  $\mathit{inst(aux_C,Young)}$ and $\mathit{inst(aux_C,NotYoung)}$ respectively; they are generated
in 0.40 seconds.

A first scalability test is based on a slightly larger version of the same example, with 5 distinguished classes and 50 typicality inclusions. Adding to it more typicality inclusions, up to 8 times (400 inclusions), the runtime grows up to 0.99 s (see Table \ref{runningtimes}, test 1a, average
running times for \emph{asprin 1.1.1} under Linux on an Intel Xeon E5-2640 @ 2.00GHz). Adding up to 8 copies of the KB (i.e., adding $\mathit{\tip(Employee') \sqsubseteq NotYoung'}$ and similar), with up to 40 distinguished classes and 400 typicality inclusions, the runtime grows up to 3.93 (Table \ref{runningtimes}, test 1b).

In another experiment, we have distinguished classes $C_1 \ldots C_5$ with 
$C_3 \sqsubseteq C_2 \sqsubseteq C_1$, $C_5 \sqsubseteq C_4 \sqsubseteq C_1$. For all $i$,
the $C_i$'s are typically $P_i$'s, $Q_i$'s, $R_i$'s, 
where for $i \neq j$, $P_i \sqcap P_j  \sqsubseteq \bot$.
A typical $C_3 \sqcap C_5$ then inherits all the $Q_i$'s and $R_i$'s properties, while it can
either be a $P_3$ or a $P_5$.
Also in this case adding up to 8 copies of the KB (with then up to 40 distinguished classes
and 120 typicality inclusions) leads to a moderate increase of the running time which ranges from 1.03 to 1.76 seconds (Table \ref{runningtimes}, test 2).
   
Dealing with longer chains of subclasses seems more challenging. For a modification of the base case of the previous example with 10 distinguished classes 
$C_{10} \sqsubseteq C_8 \sqsubseteq C_6 \sqsubseteq C_4 \sqsubseteq C_2 \sqsubseteq C_1$, 
$C_9 \sqsubseteq C_7 \sqsubseteq C_5 \sqsubseteq C_3 \sqsubseteq C_1$, 
and 50 typicality axioms, checking the properties of typical $C_9 \sqcap C_{10}$ already takes 5.4 seconds.

\begin{table}[t]
	\begin{tabular}{ c || c  c  c  c }
		\hline
		\hline
		\ \  & 1x &    2x     &  4x      &  8x     \\
		\hline
		\hline
		test 1a ~ ~ 
		& 0.35
		& 0.45
		& 0.63
		& 0.99
		\\
		\hline
		\hline
		test 1b ~ ~
		& 0.35
		& 0.50
		& 0.95
		& 3.93
		\\
		\hline
		\hline
		test 2 ~ ~ ~
		& 1.03
		& 1.15
		& 1.27
		& 1.76
		\\
		\hline
		\hline
\end{tabular}
	\caption{Some scalability results}
	\label{runningtimes}
\end{table}

   \section{Conclusions and related work}  \label{sec:conclu} 

In this paper we have developed an ASP approach for defeasible inference in a concept-wise multipreference extension of $\el$.  
Our semantics is related to
 the multipreference semantics for $\alc$ developed by Gliozzi \shortcite{GliozziAIIA2016}, which is based on the idea of refining the rational closure construction considering the preference relations $<_{A_i}$ associated with different aspects, but we follow a different route concerning both the definition of the preference relations associated with concepts, and  the way of combining them in a single preference relation. In particular, Gliozzi's multipreference semantics aims at defining a refinement of rational closure semantics, which is not our aim here;
 compared with rational closure, our semantics is neither weaker (as it does not suffer from the ``the blocking of property inheritance" problem) nor stronger (see Section 4).

The idea of having different preference relations, associated with different typicality operators, has been studied by Gil \shortcite{fernandez-gil} to define a multipreference formulation of the description logic $\alctmin$,  
a typicality DL with a minimal model preferential semantics. 
In this proposal we associate preferences with concepts,  
and we  combine such preferences into a single global one. 
For a preferential extension of $\el$ based on the same minimal model semantics as $\alctmin$, it has been proved \cite{ijcai2011} that minimal entailment is already \textsc{ExpTime}-hard for $\el$ KBs,  
while a $\Pi^p_2$ upper bound holds for minimal entailment in the Left Local fragment of $\eltm$,
as for circumscriptive KBs  \cite{Bonatti2011}.
A related problem of commonsense concept combination has also been addressed in a probabilistic extension of $\alctr$ 
 \cite{Lieto2018}.

Among the formalisms combining DLs with logic programming rules \cite{Eiter2008,Eiter2011,rosatiacm,KnorrECAI12,Gottlob14} 
DL-programs \cite{Eiter2008,Eiter2011} support a loose coupling of DL ontologies and rule-based reasoning
under the answer set semantics and the well-founded semantics; rules may contain DL-atoms in their bodies,
corresponding to queries to a DL ontology, which can be modified according to a list of updates.
The non-monotonic description logic ${\cal DL}^N$ \cite{bonattiAIJ15} supports normality concepts  based on a notion of overriding,
enjoying good computational properties, and preserves the tractability for low complexity DLs, including ${\el}^{++}$ and $DL$-$lite$ \cite{BonattiSWC15}.
Bozzato et al. \shortcite{Bozzato14,Bozzato2018} present extensions of the CKR (Contextualized Knowledge Repositories) framework in which defeasible axioms are allowed in the global context 
and exceptions can be handled by overriding and have to be justified in terms of semantic consequence.
A translation of extended CKRs (with  knowledge bases in ${\cal SROIQ}$-RL) into Datalog programs under the answer set semantics is developed. 
Related approaches are also the work by  Beierle et al. \shortcite{BeierleAMAI2018}, characterizing skeptical c-inference as a constraint satisfaction problem,
and the work by Deane et al. \shortcite{Russo2015} presenting an inconsistency tolerant semantics for ${\cal ALC}$ using preference weights and 
exploiting ASP optimization for computing preferred interpretations.
Reasoning under the rational closure for  low complexity DLs  
has been investigated for $\sroel$ \cite{SROEL_FI_2018}, using a Datalog plus stratified negation polynomial construction  and for ${\cal ELO}_{\bot}$ \cite{CasiniStracciaM19}, developing a polynomial time subsumption algorithm for the nominal safe fragment \cite{IncredibleELK_JAR14}.
A problem that we have not considered in this paper is the treatment of defeasible information for existential concepts;
it has been addressed by Pensel and Turhan \shortcite{Pensel18}, who developed a stronger version of rational and relevant entailment in $\el$, 
exploiting  a materialisation-based algorithm for $\el$ and a canonical model construction.

It is known that Brewka's $\#$ strategy \shortcite{Brewka04} 
exploits the lexicographical order also used by Lehmann to define the models of the lexicographic closure of a conditional knowledge base \cite{Lehmann95}, starting from the rational closure ranking. 
This suggests that, while we have used this strategy for ranked TBox ${\cal T}_{C_j}$ containing only  typicality inclusions of the form $\tip(C_j) \sqsubseteq D$,
coarsely grained ranked TBoxes could be allowed, in which ${\cal T}_{C_j}$ contains all typicality inclusions 
$\tip(E) \sqsubseteq D$ for any subclass $E$ of $C_j$. We expect that this might improve performances, by reducing the number of $\leq_{C_j}$ relations to be considered.
We leave for future work investigating whether our ASP approach with preferences can be used for computing the lexicographic closure for $\elpb$,
and whether alternative notions of specificity can be adopted.

The modular separation of the typicality inclusions in different TBoxes and their separate use for defining preferences $\leq_{C_i}$ 
 suggests that some of the optimizations used by ELK reasoning algorithms \cite{IncredibleELK_JAR14} might be 
 extended to our setting.

{\bf Acknowledgement:} We thank the anonymous referees for their helpful comments and suggestions. This research is partially supported by INDAM-GNCS Project 2019. 


\begin{appendix}

\section{Proofs for Sections 3 and 4}

\begin{proposition} \label{prop:existence-Tcompliant}
Let $K$ be a ranked knowledge base over ${\cal C}$. If $K$ has a preferential model, then there is an $\elpb$ interpretation $\tip$-compliant for $K$.
\end{proposition}
\begin{proof}[{ Proof(sketch)}]
Let 
$K=\langle  {\cal T}_{strict},$ $ {\cal T}_{C_1}, \ldots,$ $ {\cal T}_{C_k}, {\cal A}  \rangle$ be a ranked $\elpb$ knowledge base over  ${\cal C}$.
Following Giordano et al.\ \shortcite{lpar2007} a preferential model for an $\elpb$ knowledge base $K$ can be defined as a triple  ${\enne}= \langle \Delta, <, \cdot^I \rangle$ such that:
$ \langle \Delta,\cdot^I \rangle$ is an $\elpb$ interpretation satisfying all inclusions in  ${\cal T}_{strict}$ and all assertions in ${\cal A}$;
$<$ is is an irreflexive, transitive, well-founded binary relation on $\Delta$; and, for all typicality inclusions $ \tip(C_j) \sqsubseteq D$ in $K$, $min_<(C_j^I) \subseteq D^I$ holds.
If $K$ has a preferential model $ \langle \Delta, <, \cdot^I \rangle$, it has also an $\elpb$ model  which is compliant with $K$. In fact, 
$I= \langle \Delta,\cdot^I \rangle$ is an $\elpb$ model of $K$ and, if $C_j^I\neq \emptyset$, by well-foundedness, $min_<(C_j^I)\neq \emptyset$ 
Hence, there is some $x \in min_<(C_j^I)$. Clearly, $x \in C_j^I$. 
As all typicality inclusions are satisfied in the preferential model $\enne$, 
for all $ \tip(C_j) \sqsubseteq D$ in ${\cal T}_{C_j}$, $min_<(C_j^I) \subseteq D^I$ holds and, hence,
$x \in D^I$. 
Thus, when $C_j^I\neq \emptyset$,  $x \in C_j^I$ and $x$ satisfies all defeasible inclusions in  ${\cal T}_{C_j}$. As this holds for all $C_j \in {\cal C}$,
the $\elpb$ interpretation $ \langle \Delta, \cdot^I \rangle$ is $\tip$-compliant for $K$.
\end{proof}

\begin{proposition} \label{prop:KLM_properties}
cw$^m$-entailment  
satisfies all KLM postulates of  preferential consequence relations.

\end{proposition} 
\begin{proof}

We prove that any cw$^m$-model $\emme=\langle \Delta, <_{C_1}, \ldots, <_{C_k}, <, \cdot^I \rangle $ of $K$
satisfies the properties of a  preferential consequence relation. We consider the following obvious reformulation of the properties, considering that $\tip(C) \sqsubseteq D$ stands for the conditional $C {\ent} D$ in KLM preferential logics \cite{KrausLehmannMagidor:90,whatdoes}:

{\em

(REFL) \ $\tip(C) \sqsubseteq C $ 

(LLE) \ If $A \equiv B$ and $\tip(A) \sqsubseteq C $, then $\tip(B) \sqsubseteq C $ 

(RW) \  If $C \sqsubseteq D$ and $\tip(A) \sqsubseteq C $, then $\tip(A) \sqsubseteq D $ 

(AND) \ If $\tip(A) \sqsubseteq C $ and $\tip(A) \sqsubseteq D $, then $\tip(A) \sqsubseteq C \sqcap D $

(OR) \ If $\tip(A) \sqsubseteq C $ and $\tip(B) \sqsubseteq C $, then $\tip(A \sqcup B) \sqsubseteq C $

(CM) \  If $\tip(A) \sqsubseteq D$ and $\tip(A) \sqsubseteq C $, then $\tip(A \sqcap D) \sqsubseteq C $ }

\noindent
We exploit the fact that   the structure $ \langle \Delta, \cdot^I \rangle$ is an $\elpb$ interpretation:
\begin{itemize}
\item
(REFL) \ $\tip(C) \sqsubseteq C $ 

We have to prove that $\emme$ satisfies (REFL) \ $\tip(C) \sqsubseteq C $ , i.e., $min_<(C^I) \subseteq C^I$, which holds by definition of $min_<$.

\item
(LLE) \ If $A \equiv B$ and $\tip(A) \sqsubseteq C $, then $\tip(B) \sqsubseteq C $. 

Here, $A \equiv B$ means that $A$ and $B$ are equivalent in the underlying logic $\elpb$, i.e., for all  $\elpb$ interpretations $J= \langle \Delta^J, \cdot^J \rangle$, $A^J= B^J$. 
If $\tip(A) \sqsubseteq C $ is satisfied in $\emme$, then $min_<(A^I) \subseteq C^I$. As $ \langle \Delta, \cdot^I \rangle$ is a $\elpb$ interpretation,
$A^I= B^I$, and $min_<(B^I) =min_<(A^I) \subseteq C^I$. Thus, $\tip(B) \sqsubseteq C $ is satisfied in $\emme$.

\item
(RW) \  If $C \sqsubseteq D$ and $\tip(A) \sqsubseteq C $, then $\tip(A) \sqsubseteq D $. 

Let us assume that $\tip(A) \sqsubseteq C $ is satisfied in $\emme$ and that  $C \sqsubseteq D$ is satisfied in all $\elpb$ interpretations.
Then: $min_<(A^I) \subseteq C^I$ and $C^I \subseteq D^I$.  Hence, $min_<(A^I) \subseteq D^I$. Therefore, $\tip(A) \sqsubseteq D $ is satisfied in $\emme$.

\item
(AND) \ If $\tip(A) \sqsubseteq C $ and $\tip(A) \sqsubseteq D $, then $\tip(A) \sqsubseteq C \sqcap D $.

Let us assume that $\tip(A) \sqsubseteq C $ and $\tip(A) \sqsubseteq D $ are satisfied in $\emme$, that is $min_<(A^I) \subseteq C^I$ and $min_<(A^I) \subseteq D^I$. For all $x \in  min_<(A^I)$,  $x \in C^I$ and $x \in D^I$ and, hence, $x \in (C \sqcap D)^I$. Therefore, $\tip(A) \sqsubseteq C \sqcap D $ is satisfied in $\emme$.

\item
(OR) \ If $\tip(A) \sqsubseteq C $ and $\tip(B) \sqsubseteq C $, then $\tip(A \sqcup B) \sqsubseteq C $

If $\tip(A) \sqsubseteq C $ and $\tip(B) \sqsubseteq C $ are satisfied in $\emme$, then $min_<(A^I) \subseteq C^I$ and $min_<(B^I) \subseteq C^I$.  Hence, 
if $x \in min_<((A \sqcup B)^I)$,  $x \in min_<(A^I)$ or $x \in min_<((B^I)$. In both cases $x \in C^I$ holds. Hence, $\tip(A \sqcup B) \sqsubseteq C $ is satisfied in $\emme$.

\item
(CM) \  If $\tip(A) \sqsubseteq D$ and $\tip(A) \sqsubseteq C $, then $\tip(A \sqcap D) \sqsubseteq C $ 

If $\tip(A) \sqsubseteq D$ and $\tip(A) \sqsubseteq C $ then $min_<(A^I) \subseteq D^I$ and $min_<(A^I) \subseteq C^I$. From $min_<(A^I) \subseteq D^I$, we can conclude that  $min_<((A \sqcap D)^I) = min_<(A^I)$. Hence, $min_<((A \sqcap D)^I) \subseteq C^I$, and $\tip(A \sqcap D) \sqsubseteq C $ is satisfied in $\emme$.

\end{itemize}

As the set of typicality inclusions satisfied by $\emme$, $Typ^\emme$, satisfies all KLM postulates above, $Typ^\emme$ is a a preferential consequence relation. Then each cw$^m$-model $\emme$ of $K$  determines a preferential consequence relation $Typ^\emme$.  
The typicality inclusions $\tip(C) \sqsubseteq D$ cw$^m$-entailed by $K$ are, by definition,  satisfied in all canonical $\tip$-compliant cw$^m$-model $\emme$ of $K$. They belong to $Typ^\emme$, for all canonical $\tip$-compliant cw$^m$-model $\emme$ of $K$. Hence, the set of all typicality inclusions cw$^m$-entailed by $K$ is the the intersection of all $Typ^\emme$, for $\emme$ a canonical $\tip$-compliant cw$^m$-model of $K$.

Kraus, Lehmann and Magidor \shortcite{KrausLehmannMagidor:90,whatdoes} 
have proved that the intersection of a set of preferential consequence relations is as well a preferential consequence relation. It follows that the set of the typicality inclusions cw$^m$-entailed by $K$ is a preferential consequence relation and, therefore, it satisfies all KLM postulates of a preferential consequence relation.
\end{proof}

\section{Materialization Calculus }
\label{Appendix_Calculus}

We report the fragment of the materialization calculus \cite{KrotzschJelia2010} which is used in Section 5. 
The representation of a knowledge base ({\em input translation}) is as follows, where, to keep a DL-like notation, we do not follow the ASP convention where variable names start with uppercase;
in particular, $A$, $B$ $C$, and $R$, $S$, $T$, are intended as ASP constants corresponding to the same class/role names in $K$:

\vspace{-0.3cm}
\begin{tabbing}
$ \mathit{ \exists R . Self \sqsubseteq C }$ \= \kill  \\
\> $\mathit{ a \in N_I }$ \'             $\mapsto \mathit{ nom(a) } $ \\
\> $\mathit{ C \in N_C }$ \'             $\mapsto \mathit{ cls(C) } $ \\
\> $\mathit{ R \in N_R }$ \'             $\mapsto \mathit{ rol(R) } $ \\
\> $\mathit{C(a)}$ \'                    $\mapsto \mathit{ subClass(a,C) } $ \\
\> $ \mathit{ R(a,b)}$ \'                $\mapsto \mathit{supEx(a,R,b,b) } $ \\
\> $ \mathit{ \top \sqsubseteq C}$ \'    $\mapsto \mathit{top(C)} $ \\
\> $ \mathit{ A \sqsubseteq \bot}$ \'    $\mapsto \mathit{bot(A) } $ \\
\> $ \mathit{ \{a\} \sqsubseteq C  }$ \' $\mapsto \mathit{subClass(a,C) } $ \\
\> $ \mathit{ A \sqsubseteq C  }$ \'     $\mapsto \mathit{subClass(A,C) } $ \\
\> $ \mathit{ A \sqcap B \sqsubseteq C }$ \'            $\mapsto \mathit{subConj(A,B,C) } $ \\
\> $ \mathit{ \exists R . A \sqsubseteq C  }$ \'        $\mapsto \mathit{subEx(R,A,C)  } $ \\
\> $ \mathit{ A \sqsubseteq \exists R . B   }$ \'    $\mapsto \mathit{supEx(A,R,B,aux_i)  }$ \\
\> $ \mathit{ R \sqsubseteq T  }$ \'    $\mapsto \mathit{ subRole(R,T)  } $ \\
\> $ \mathit{ R \circ S  \sqsubseteq T  }$ \'    $\mapsto \mathit{ subRChain(R,S,T)  } $ \\
\end{tabbing}
\vspace{-0.3cm}
In the translation of $ \mathit{ A \sqsubseteq \exists R . B   }$,
$\mathit{aux_i}$  is a new constant, and a different constant is needed for each axiom of this form.

The {\em inference rules} included in $\Pi_{K}$ in Section 5 are the following\footnote{Here, $u, v, x, y, z, w$, possibly with suffixes, are ASP variables.}:
\vspace{-0.3cm}
\begin{tabbing}
$(10)$ \= \kill \\
$(1) ~ \mathit{inst(x, x) \leftarrow nom(x) }  $\\
$(3) ~ \mathit{inst(x, z) \leftarrow top(z), inst(x, z') } $\\
$(4) ~ \mathit{inst(x, y) \leftarrow bot(z), inst(u, z), inst(x, z'), cls(y) } $\\
$(5) ~ \mathit{inst(x, z) \leftarrow subClass(y, z), inst(x, y) }  $\\
$(6) ~ \mathit{inst(x, z) \leftarrow subConj(y1, y2, z), inst(x, y1), inst(x, y2) } $\\
$(7) ~ \mathit{inst(x, z) \leftarrow subEx(v, y, z), triple(x, v, x'), inst(x', y) } $\\
$(9) ~ \mathit{triple(x, v, x') \leftarrow supEx(y, v, z, x'), inst(x, y) } $\\
$(10) ~ \mathit{inst(x', z) \leftarrow supEx(y, v, z, x'), inst(x, y) } $\\
$(13) ~ \mathit{triple(x, w, x') \leftarrow subRole(v, w), triple(x, v, x') } $ \\
$(15) ~ \mathit{triple(x, w, x'') \leftarrow subRChain(u, v,w), triple(x, u, x'), triple(x', v, x'') } $ \\
$(27) ~ \mathit{inst(y, z) \leftarrow  inst(x, y), nom(y), inst(x, z) } $ \\
$(28) ~ \mathit{inst(x, z) \leftarrow  inst(x, y), nom(y), inst(y, z) } $ \\
$(29) ~ \mathit{triple(z, u, y) \leftarrow inst(x, y), nom(y), triple(z, u, x) } $
\end{tabbing}

We include the additional rule: \\
$(4b) ~ \mathit{\bot \leftarrow bot(z), inst(u, z) } $

\section{Proofs for Section 5} \label{appendix:Section4}

\subsection{Proof of Proposition \ref{AS to models}} \label{appendix:Prop2}

{\em Proposition \ref{AS to models}}
{
Given a normalized ranked knowledge base $K=\langle  {\cal T}_{strict},$ $ {\cal T}_{C_1}, \ldots,$ $ {\cal T}_{C_k}, {\cal A}  \rangle$ over the set of concepts ${\cal C}$,  
and a (normalized) subsumption  $C \sqsubseteq D$, we can prove the following: 
\begin{itemize}
\item[(1)]
if there is an answer set $S$ of the ASP program $\Pi(K,C,D)$, 
such that $inst(aux_C, D) \not \in S$,
then there is a $\tip$-compliant cw$^m$-model $\emme=\langle \Delta, <_{C_1}, \ldots, <_{C_k}, <, \cdot^I \rangle $ for $K$  
that falsifies the subsumption  $C \sqsubseteq D$.

\item[(2)]
if there is a $\tip$-compliant cw$^m$-model $\emme=\langle \Delta, <_{C_1}, \ldots, <_{C_k}, <, \cdot^I \rangle $ of $K$  
that falsifies the subsumption  $C \sqsubseteq D$,
then there is an answer set $S$ 
of $\Pi(K,C,D)$,  
such that $\mathit{ inst(aux_C, D) \not \in S}$.
\end{itemize}
}

\medskip
\noindent
For part (1), in the following, given the answer set $S$ of the program $\Pi(K,C,D)$ 
such that $inst(aux_c, C) \not \in S$,  we construct a cw$^m$-model $\emme$ falsifying $C \sqsubseteq D$.
In particular, we construct the domain of $\emme$ from the set $Const$ including all the named constants $c \in N_I$
as well as all the auxiliary constants occurring in the ASP program $\Pi(K,C,D)$ 
defining an equivalence relation over constants and using equivalence classes
to define domain elements. The construction is similar to the proof of completeness by Kr\"{o}tzsch  \shortcite{jeliaReport}. For readability, we write $\auxARC$ and $aux_{C_i}$, respectively, for the constants
associated with inclusions $A \sqsubseteq \exists R.C$ and with the typicality concepts $\tip(C_i)$, for a concept  $C_i \in {\cal C}$.
Observe that the answer set $S$ contains {\em no information} about the definition of the preference relations $<_{C_j}$'s and $<$ on domain elements that can be used to build the model $\emme$.
In fact, these relations are not encoded in the ASP program. They are only used in \emph{asprin} for defining the preference among models.

First, let us define a relation $\approx$ between the constants in $Const$:

\medskip

\noindent
{\em Definition 7 }

\noindent
Let $\approx$ be the reflexive, symmetric and transitive closure of the relation
$\{ (c,d) \mid inst(c,d) \in S$, for $c \in Const$ and  $d \in N_I\}$.

\medskip

\noindent
It can be proved that:
\begin{lemma} \label{lemma:approx}
Given a constant $c$ such that $c \approx a$ for $a \in N_I$,
if  $inst(c,A)$ ($\mathit{triple(c,R,d), triple(d,R,c)}$) is in $S$,
then $inst(a,A)$ ($\mathit{triple(a,R,d), triple(d,R,a)}$) is in $S$.
\end{lemma}
The proof is similar to the proof of Lemma 2 by Kr\"{o}tzsch \shortcite{jeliaReport}.
Vice-versa it holds that:
\begin{lemma} \label{lemma:approx2}
Given a constant $c$ such that $c \approx a$ for $a \in N_I$,
if  $inst(a,A)$ ($triple(a,R,d)$) is in $S$,
then $inst(c,A)$ ($triple(c,R,d)$) is in $S$.
\end{lemma}

Now, let $[c]=\{ d \mid d \approx c \}$ denote the equivalence class of $c$;
we define the domain $\Delta$ of the interpretation $\emme$ as follows:
 $\Delta=  \{[c] \mid c \in N_I\} \cup \{\wARC \mid$ $\mathit{inst(\auxARC,e) \in S}$ for some $e$ and there is no $d \in N_I$ such that $\auxARC \approx d\}$
 $ \cup \{z_{C_i} \mid$ $\mathit{inst( aux_{C_i},e) \in S}$ for some $e$ and there is no $d \in N_I$ such that $aux_{C_i} \approx d\}$
 $ \cup \{z_C \mid$ there is no $d \in N_I$ such that $aux_{C} \approx d\}$\footnote{We need a single copy of auxiliary constants as, unlike the original calculus \cite{jeliaReport}, we do not handle $\mathit{Self}$ statements.}.

For each element $e \in \Delta$, we define a projection $\prj(e)$ to $Const$ as follows:

- $\prj([c])=c$;

- $\prj(\wARC)=\auxARC$;

- $\prj(z_{C_i})=aux_{C_i}$;

- $\prj(z_{C})=aux_{C}$.

\noindent
The interpretation of individual constants, concepts and roles over $\Delta$ is defined as follows:
\begin{quote}
- for all $c \in N_I$, \ $c^I =[c]$;

 - for all $d \in \Delta$, $A \in N_C$, \   $d \in A^{I}$ iff $\mathit {inst}(\prj(d), A) \in S$; 
 
- for all $d, e \in \Delta$, \ $(d,e) \in R^{I}$ iff ($\mathit{triple}(\prj(d), R, \prj(e)) \in S$).
\end{quote}
This defines an $\elpb$ interpretation.

The the relations $\leq_{C_j}$ on $\Delta$ are defined according to Definition \ref{total_preorder}, which gives a total preorders.
Then, the preference relation $<$ on elements in $\Delta$ is defined according to Definition \ref{def-multipreference-int}, point (c).
Observe that $d_1 \leq_{C_J} d_2$ only depends on which typicality inclusions in ${\cal T}_{C_j}$ are satisfied by $d_1$ and $d_2$
and this only depends on the interpretation of concept names and individual names defined above. Similarly for $<$.

For all concepts $C_j$ such that $C_j^I \neq \emptyset$, there is an element $d \in C_j^I$ and, by construction, $\mathit {inst}(\prj(d), C_j) \in S$.
By rules (b) and (c), $\mathit {inst}(aux_{C_j}, C_j) \in S$ and $\mathit {typ}(aux_{C_j}, C_j) \in S$. 
By rule (d), $aux_{C_j}$ must satisfy all the typicality inclusions for $C_j$. 
Hence, when $C_j^I \neq \emptyset$,  $z_{C_i}$ satisfies in $\emme$ all  the typicality inclusions for $C_j$, 
a condition that is needed for $\emme$ to be $\tip$-compliant  for $K$ (together with the condition that $\emme$ satisfies all strict inclusions, see below).
Furthermore, all elements
$d \in min_{<_{C_i}}(C_i^I)$ must satisfy as well all typicality inclusions for $C_j$ as, like $z_{C_i}$, $d$ must be instance of all concepts  $D$ such that $\tip(C_j) \sqsubseteq D \in {\cal T_{C_j}}$ (otherwise $z_{C_i} <_{C_i} d $ would hold). As a consequence, for all $d \in min_{<_{C_i}}(C_i^I)$, $d \sim_{C_i} d'$.

It is easy to verify that $\emme$ is a $\tip$-compliant cw$^m$-model of $K$.
By construction, the $<_{C_j}$'s are defined according to Definition \ref{total_preorder} 
and $<$ is defined according to Definition \ref{def-multipreference-int}, and satisfy all the properties of preference relations in a cw$^m$-interpretation.
We have to prove that $\emme$  satisfies  all strict inclusions inclusions in $ {\cal T}_{strict}$ and assertions in ${\cal A}$.
As assertions $A(c)$ ($R(a,b)$) are represented by concept inclusions $\{c\} \sqsubseteq A$ (resp. $\{c\} \sqsubseteq A$),
it suffices to prove that all the strict inclusions in $K$ are satisfied in $\emme$.
This can be done by cases, as in  Lemma 2 by Kr\"{o}tzsch \shortcite{jeliaReport}, considering all (normalized) axioms which may occur in $K$
(which are a subset of those admitted in a  $\sroel$ knowledge base).

Hence, $\emme$ is a $\tip$-compliant cw$^m$-model of $K$.
We prove that $\emme$ falsifies $C \sqsubseteq D$, that is, $C^I \not \subseteq D^I$.
Program $\Pi(K,C,D)$ contains $inst(aux_C, C)$ and, hence, $inst(aux_C, C)\in S$. From the hypothesis, $inst(aux_C, D) \not \in S$.
Hence, if $z_C \in \Delta$, by construction, $z_C  \in C^I$ and  $z_C \not \in D^I$.
Otherwise, if  $z_C \not \in \Delta$, $aux_C \approx d$, for some $d \in N_I$.
In such a case, by Lemmas  \ref{lemma:approx} and  \ref{lemma:approx2} above, $inst(d, C) \in S$ and $inst(d, D) \not \in S$.
By construction of $\emme$, $[d] \in C^I$ and $[d] \not \in D^I$. Hence, $C \sqsubseteq D$ is falsified in $\emme$.

\medskip 
For part (2), assume that there is a $\tip$-compliant cw$^m$-model $\emme=\langle \Delta, <_{C_1}, \ldots, <_{C_k}, <, \cdot^I \rangle $ of $K$  
that falsifies the subsumption  $C \sqsubseteq D$. Then there is an element $x \in \Delta$ such that $x \in C^I$ and $x \not \in D$.
We construct a $\tip$-compliant answer set $S$
of the ASP program of $\Pi(K,C,D)$, 
such that $inst(aux_C, C) \in S$ and $inst(aux_C, D) \not \in S$ (notice that $inst(aux_C, C)$ is already a fact in $\Pi(K,C,D)$).

Let us consider which typical properties are satisfied by $x$ in $\emme$, and build the following set of facts: 
\begin{quote}
$F_0= \{ \mathit{inst(aux_C,B)} \mid \; x \in B^I$ for some $\tip(C_j) \sqsubseteq B$ is in $K \}$ 
\end{quote}
and a corresponding set of DL assertions:
\begin{quote}
${\cal A}_0= \{ \mathit{B(aux_C))} \mid \; x \in B^I$ for some $\tip(C_j) \sqsubseteq B$ is in $K \}$
\end{quote}
assuming $aux_C$ is a new individual name added in $N_I$.

Notice that $F_0$ corresponds to a possible choice of properties for $aux_C$ according to rule (a).
Let us consider the Datalog program  $\Pi'(K_{strict},C,D)$ representing the encoding of the strict part of the knowledge base $K$,  
obtained by removing rules (a-d) from $K_{IR}$ and typicality inclusions from the input translation of $K$, 
while adding facts in $F_0$.
$\Pi'(K,C,D) = (\Pi(K_{Strict},C,D) - \{\mbox{(a-d)}\}) \cup F_0$.

We can prove that $\Pi'(K,C,D)$ is consistent,  i.e. $\Pi'(K,C,D)$ does not derives  $\mathit{inst(d,E)}$ for any $E$ such that $E \sqsubseteq \bot$ is in $K$.
By contradiction, if $\Pi'(K,C,D)$ is inconsistent, it derives (in Datalog) $\mathit{inst(d,E)}$ for some $E$ such that $E \sqsubseteq \bot$ is in $K$.
By Lemma 1 of Kr\"{o}tzsch \shortcite{jeliaReport}, we would get:
$K_{strict} \cup {\cal A}_0 \models \kappa(d) \sqsubseteq E$ in $\elpb$, where $\kappa(d)$ is a concept expression associated with the Datalog constant $d$  (for instance, for $d \in N_I$,  $\kappa(d)= \{d\}$). This would imply that $K_{strict} \cup {\cal A}_0$ is inconsistent. However, as $\emme$ satisfies ${\cal T}_{strict} $,  a model $\emme'=\langle \Delta, <_{C_1}, \ldots, <_{C_k}, <, \cdot^{I'} \rangle $ for ${\cal T}_{strict} \cup {\cal A}_0$ can be constructed from $\emme$ by interpreting the new individual name $anx_C$ occurring in ${\cal A}_0$ as $aux_C^{I'}=x$ (while, for the rest, $\emme'$ is defined exactly as $\emme$).
This contradicts the assumption that $\Pi'(K,C,D)$ is inconsistent.

If, for some constant $a$,  $\mathit{inst(a,C_j)}$ is derivable from $\Pi'(K,C,D)$, by Lemma 1 of
Kr\"{o}tzsch  \shortcite{jeliaReport}, any model of 
$K_{strict} \cup {\cal A}_0$ satisfies $\kappa(a) \sqsubseteq C_j$ (with $\kappa(a)$ non-empty).
Hence, in all $\elpb$ models of ${\cal T}_{strict} \cup {\cal A}_0$ the interpretation of concept $C_j$ cannot be empty. 
As we aim at building an answer set $S$ that captures $\tip$-compliance, 
we consider a new set of facts $F_1$ containing, for each constant $aux_{C_i}$ (representing a typical $C_j$-element, i.e.,
a minimal $C_j$-element wrt. $<_{C_j}$), the fact $\mathit{inst(aux_{C_j},C_j)}$, as well as the facts $ \mathit{inst(aux_{C_j},B)} $ to represent the typical properties $B$ of $<_{C_j}$-minimal $C_j$-elements.
Let us consider the following set of facts:
\begin{quote}
$F_1= \{ \mathit{inst(aux_{C_j},C_j)} \mid \; C_j \in {\cal C}, \;  $ such that $\Pi'(K,C,D) \vdash \mathit{inst(a,C_j)}$, for some constant $a\} \; \cup$\\
$\; \mbox{\ \ \ \ \ \ \ \ \  }$ $  \{ \mathit{inst(aux_{C_j},B)} \mid \; C_j \in {\cal C}, \;  $ such that $\Pi'(K,C,D) \vdash \mathit{inst(a,C_j)}$
and  $\tip(C_j) \sqsubseteq B \in K \}$.
\end{quote}
and a corresponding set of DL assertions:
\begin{quote}
${\cal A}_1= \{ \mathit{C_j(aux_{C_j})} \mid \; C_j \in {\cal C}, \;  $ and $ \mathit{inst(aux_{C_j},C_j)} \in F_1 \} \; \cup$ \\
$\; \mbox{\ \ \ \ \ \ \ \ \  }$ \;$  \{ \mathit{B(aux_{C_j})} \mid \; C_j \in {\cal C}, \;  $ and $\mathit{inst(aux_{C_j},C_j)} \in F_1$ and  $\tip(C_j) \sqsubseteq B \in K \}$.
\end{quote}
where we let $aux_{C_j}$ to be a new individual name in $N_I$.
Again, we can prove that the extended Datalog program $\Pi'(K,C,D) \cup F_1$ is consistent. If not, using a similar argument as before, we would conclude that the knowledge base $K_{strict} \cup {\cal A}_0 \cup {\cal A}_1$ is inconsistent, which is not true. In fact, as observed above, if  $\mathit{inst(a,C_j)}$ is derivable from $\Pi'(K,C,D)$,   in all $\elpb$ models of ${\cal T}_{strict} \cup {\cal A}_0$ the interpretation of concept $C_j$ cannot be empty.
In particular, $\emme'$ is a model of $K_{strict} \cup {\cal A}_0$ and $\emme'$, like $\emme$, is $\tip$-compliant for $K$. Thus, for each $C_j$ such that $\mathit{inst(a,C_j)}$ is derivable from $\Pi'(K,C,D)$, $C_j^{I'} \neq \emptyset$
and, by $\tip$-compliance, there must be some $y_j \in \Delta$ such that 
$y_j \in C_j^{I'}$ and $y_j$ satisfies all  defeasible inclusions in ${\cal T}_{C_j}$. We can therefore build a model $\emme''$ of $K_{strict} \cup {\cal A}_0 \cup {\cal A}_1$ using the domain element $y_j$ as the interpretation of the individual name 
$aux_{C_j}$ (letting $aux_{C_j}^{I''}=y_j$).  For the rest, $\emme''$ is defined as $\emme'$.

We can further add to $\Pi'(K,C,D) \cup F_1$ the set of facts:
\begin{quote}
$F_2= \{ \mathit{typ(aux_{C_j},C_j)} \mid \; C_j \in {\cal C}, \;  $ such that $\Pi'(K,C,D) \vdash \mathit{inst(a,C_j)}$ for some $a\}$,
\end{quote}
where $ \mathit{typ(aux_{C_j},C_j)} $ makes it explicit that $aux_{C_i}$ is a typical instance of $C_j$ (wrt. $<_{C_j}$).
As the predicate $\mathit{typ}$ does not occur in $\Pi'(K,C,D) \cup F_1$, the program $\Pi'(K,C,D) \cup F_1 \cup F_2$ is still consistent.

Similarly, if we add to $\Pi'(K,C,D)$ also the encoding of the typicality inclusions in the knowledge base $K$ by facts $\mathit{subTyp(C_j,B,N)}$
in $\Pi_K$, then the resulting program $\Pi''(K,C,D) \cup F_1 \cup F_2$ is still consistent. In fact, predicate $\mathit{subTyp}$ only occurs in the input translation. 

Let $S$ be the set of all ground facts which are derivable in Datalog from program $\Pi''(K,C,D)$ $ \cup F_1 \cup F_2$.
The ground instances of all rules in $\Pi''(K,C,D)$ are clearly satisfied in $S$. 
Rules (a)-(d) are the only rules in $\Pi(K,C,D)$ that do not belong to $\Pi''(K,C,D)$.
The ground instances of rules (a)-(d) (over the Herbrand Universe of program $\Pi''(K,C,D) \cup F_1 \cup F_2$) are satisfied in $S$.
For rule (a), for any $C_j \in {\cal C}$, the literals in $F_0$ make the head of disjunctive rule (a) satisfied in $S$, for each $B$ such that $\tip(C_j) \sqsubseteq B \in K$.
For rule (b), by construction, if $\mathit{inst(a,C_j)} \in S$, for some constant $a$, then $\mathit{inst(aux_{C_j},C_j)}$ is in $F_1$ and then in $S$.
For (c), $\mathit{typ(aux_{C_j},C_j)}$ is in $F_2$ (and hence in $S$) if $\mathit{inst(aux_{C_j},C_j)}$ is in $F_1$.
For (d), if $\mathit{typ(aux_{C_j},C_j)}$ is in $S$, then  $\mathit{typ(aux_{C_j},C_j)}$ is in $F_2$ and  $\mathit{inst(aux_{C_j},C_j)}$ is in $F_1$.
Then, by construction, $F_1$  must also contain $\mathit{inst(aux_{C_j},B)}$ for each  $\tip(C_j) \sqsubseteq B$ in $K$ (i.e., for each $\mathit{subTyp(C_j,B,N)}$ in $\Pi_K$).

We have proved that $S$ is a consistent set of ground atoms and all the ground instances of the rules in $\Pi(K,C,D)$ are satisfied in $S$.
To see that $S$ is an answer set of $\Pi(K,C,D)$, we have to show that all literals in $S$ are supported in $S$.
Just observe that, all facts in $F_0$ are obtained applying the disjunctive rule (a). Facts $\mathit{inst(aux_{C_j},C_j)}$ in $F_1$ and $\mathit{typ(aux_{C_j},C_j)}$ in $F_2$ can be derived using rules (b) and (c) from facts ($\mathit{inst(a,C_j)}$) derived from $\Pi(K,C,D) \cup F_0$ (by construction of $F_1$ and $F_2$). The facts in $\mathit{inst(aux_{C_j},B)}\in F_1$ for each  $\tip(C_j) \sqsubseteq B$ in $K$ are derived using rules (d) from $\mathit{subTyp(C_j,B,N)}$ in $\Pi_K$ and from facts $\mathit{inst(aux_{C_j},C_j)}$ in $F_1$, which are supported in $S$.

To conclude the proof, we need to show that $\mathit{inst(aux_{C},D)}\not \in S$. Observe that the properties satisfied by $\mathit{aux_C}$ are given by the set of literals $F_0$, and $\mathit{inst(aux_{C},B)}$ iff $x \in B^I$ in model $\emme$. As from the hypothesis $x \not \in D^I$, then  $\mathit{inst(aux_{C},D)}$ is not in $F_0$. If $\mathit{inst(aux_{C},D)}$ were in $S$, by construction, it would be derivable from $\Pi'(K,C,D) \cup F_1$.
However, in this case, by Lemma 1 of Kr\"{o}tzsch \shortcite{jeliaReport}, any $\elpb$ model of 
$K_{strict} \cup {\cal A}_0 \cup {\cal A}_1$ would satisfy $\{aux_{C}\} \sqsubseteq D$.
This contradicts the fact that model $\emme''$ satisfies $K_{strict} \cup {\cal A}_0 \cup {\cal A}_1$ and interprets the individual name $aux_{C}$ as $(aux_{C})^{I''}=(aux_{C})^{I'}=x$.
From the hypothesis,   $x \not \in D^I$ in $\emme$ and,  hence,  $x \not \in D^{I''}$ in $\emme''$.

\medskip

\noindent
{\em Proposition \ref{AS to pref-models}.}
{Given a normalized ranked knowledge base $K=\langle  {\cal T}_{strict},$ $ {\cal T}_{C_1}, \ldots,$ $ {\cal T}_{C_k}, {\cal A}  \rangle$ over the set of concepts ${\cal C}$, 
and a subsumption  $\tip(C) \sqsubseteq D$, we can prove the following:
\begin{itemize}
\item[(1)]
if there is a canonical and $\tip$-compliant cw$^m$-model $\emme=(\Delta, <_{C_1}, \ldots, <_{C_k}, <, \cdot^I)$ of $K$  
that falsifies  $\tip(C) \sqsubseteq D$, then there is a preferred answer set $S$ 
of $\Pi(K,C,D)$ according to  $\mathit{Pref}$,
such that $inst(aux_C, D) \not \in S$.

\item[(2)]
if there is a preferred answer set $S$ of $\Pi(K,C,D)$  according to  $\mathit{Pref}$,  such that $inst(aux_C, D) \not \in S$,
then there is a canonical and $\tip$-compliant  cw$^m$-model $\emme=(\Delta, <_{C_1}, \ldots, <_{C_k}, <, \cdot^I)$ of $K$  
that falsifies  $\tip(C) \sqsubseteq D$.
\end{itemize}
}

\begin{proof}
For part (1), assume that there is a canonical $\tip$-compliant cw$^m$-model $\emme=(\Delta, <_{C_1}, \ldots, <_{C_k}, <, \cdot^I)$ of $K$  
that falsifies the subsumption  $\tip(C) \sqsubseteq D$. Then, there is  some $x \in \Delta$, such that $x \in min_<(C^I)$ and $x \not \in D^I$.

By Proposition \ref{AS to models}, part (2), we know that there is an answer set $S$ 
of program $\Pi(K,C,D)$  
such that $inst(aux_C, D) \not \in S$. 
We have to prove that $S$ is a preferred answer set of $\Pi(K,C,D)$.

Observe that, as  $\mathit{inst(aux_C,C) \in \Pi(K,C,D)}$, $\mathit{inst(aux_C,C)}$ is in  $S$. By construction of $S$, the answer set $S$ contains the set of facts $F_0$ where:
$F_0= \{ \mathit{inst(aux_C,B)} \mid \; x \in B^I$ for some $\tip(C_j) \sqsubseteq B$ is in $K \}$.

Suppose, by contradiction that $S$ is not preferred among the answer sets of $\Pi(K,C,D)$.
Then there is another answer set $S'$ which is preferred to $S$. By this we mean that $aux_C$ is $S'$ (let us denote it as $aux_C^{S'}$) is globally preferred to  $aux_C$ is $S$ (let us denote it as $aux_C^S$), that is, $aux_C^{S'} < aux_C^S$. 
We have seen that the relation of global preference among the ASP constants $aux_C^{S'}$ and $ aux_C^S$ is, essentially, the same as in Definition \ref{def-multipreference-int}, point (c), andt relations $<_{C_j}$'s and $\leq_{C_j}$
are defined according to Definition \ref{total_preorder}, by letting: 
\begin{quote}
$\mathit{{\cal T}^l_{C_i}(aux_C^{S'})}$ $=\{ \mathit{B \mid \;  inst(aux_C, C_i) \in S'}$ or $\mathit{inst(aux_C, B) \in S' }$  s.t. $\tip(C_i) \sqsubseteq B \in K \}$   and

$\mathit{{\cal T}^l_{C_i}(aux_C^S)}$ $=\{ \mathit{B \mid \; inst(aux_C, C_i) \in S}$ or $\mathit{inst(aux_C, B) \in S}$  s.t. $\tip(C_i) \sqsubseteq B \in K  \}$.
\end{quote}

By Proposition \ref{AS to models}, part (1), from the answer set $S'$ we can construct  a $\tip$-compliant cw$^m$-model $\emme'=(\Delta', <'_{C_1}, \ldots, <'_{C_k}, <', \cdot^{I'})$ of $K$
that contains a domain element $z_C \in \Delta'$ such that $z_C \in C^{I'}$ and $z_C \in B^{I'}$, for all $\mathit{\tip(C_i) \sqsubseteq B \in K}$ such that $\mathit{inst(aux_C, B) \in S' }$.

As $\emme$ is a canonical model, there must be an element $y \in \Delta$ such that $y\in B^I$ iff $z_C \in B^{I'}$, for all concepts $B$.
Therefore, 
for all $\mathit{\tip(C_i) \sqsubseteq B \in K}$, $y\in B^I$ iff $\mathit{inst(aux_C, B) \in S' }$.
We have already observed that: 
for all $\mathit{\tip(C_i) \sqsubseteq B \in K}$, $x\in B^I$ iff $\mathit{inst(aux_C, B) \in S }$.
From $aux_C^{S'} < aux_C^S$, it follows that $y < x$, which contradicts the hypothesis that $x$ is a $<$-minimal $C$-element in $\emme$.
Then, $S$ must be preferred among the answer sets of $\Pi(K,C,D)$.

\medskip

For part (2), let us assume that there is a preferred answer set $S$ of $\Pi(K,C,D)$ such that $inst(aux_C, D) \not \in S$.
By Proposition \ref{AS to models}, point (1), from the answer set $S$ we can construct  a $\tip$-compliant cw$^m$-model $\emme^*=(\Delta^*, <^*_{C_1}, \ldots, <^*_{C_k}, <^*, \cdot^{I^*})$ of $K$ in which there is some domain element $z_C \in \Delta$ such that $z_C \in B^{I^*}$ iff $\mathit{inst(aux_C, B) \in S }$.
In particular, $z_C \in C^{I^*}$ and $z_C \not \in D^{I^*}$. 

As $\emme$ is a preferential model for $K$, $K$ is consistent, and there is a canonical preferential model $\enne=$ $\langle \Delta, {<^{\enne}}, \cdot^I \rangle $ for $K$ \cite{GiordanoFI2018}. As $\enne$ is canonical, there must be an element $y \in \Delta$ such that $y\in B^I$ iff $z_C \in B^{I}$, for all $\elpb$ concepts $B$. 
Therefore, $y\in B^I$ iff $\mathit{inst(aux_C, B) \in S }$.
In particular, $y \in C^I$ and $y \not \in D^I$. 
$(\Delta, , \cdot^I)$ is a $\elpb$ model of $K$, which is canonical and $\tip$-compliant for $K$ (as any preferential model is $\tip$-compliant). 
Let $\emme=\langle \Delta, <_{C_1}, \ldots, <_{C_k}, <, \cdot^I \rangle$ be the cw$^m$-model obtained by adding to $\enne$ the preference relations in $\Delta$ defined according to Definition \ref{def-multipreference-int}, point (c), and Definition  \ref{total_preorder}.  

To conclude that $\tip(C) \sqsubseteq D$ is falsified in a canonical $\tip$-compliant  cw$^m$-model $\emme$. We have to prove that  $y \in \tip(C)^I$, that is
$y \in min_<(C^I)$. 

For a contradiction, let us assume that $y \not \in min_<(C^I)$. Then there is some $w \in \Delta$ such that $w \in min_<(C^I)$ and $w<y$. 
Given model $\emme$ and $w \in \Delta$, we can build an answer set $S'$ of $\Pi(K,C,D)$ (as done in the proof of Proposition  \ref{AS to models}, part (2),  such that
$\mathit{inst(aux_C,B)} \in S'$ $\iff  w \in B^I$, for all $B$ such that $\tip(C_j) \sqsubseteq B$ is in $K$.

As $w<y$, it must be that the answer set $S'$ is preferred to the answer set $S$ of $\Pi(K,C,D)$, 
as $aux_C$ is $S'$ (let us denote it as $aux_C^{S'}$) is globally preferred to  $aux_C$ is $S$ (let us denote it as $aux_C^S$), that is, $aux_C^{S'} < aux_C^S$. However, this contradicts the hypothesis that $S$ is a preferred answer set. Then we conclude that $y \in min_<(C^I)$. As  $y\in (\tip(C^I)$ and $y \not \in  D^I$, $y$ falsifies the subsumption  $\tip(C) \sqsubseteq D$.
\end{proof}

\begin{proposition} \label{prop: upper bound}
 cw$^m$-entailment is  in  $\Pi^p_2$.
 \end{proposition}

\begin{proof}
We consider the complementary problem, that is, the problem of deciding whether $\tip(C) \sqsubseteq D$ is not cw$^m$-entailed by $K$.
It requires determining if there is a canonical $\tip$-complete cw$^m$-model of $K$ falsifying $\tip(C) \sqsubseteq D$ or, equivalently (by Proposition \ref{AS to pref-models}), there is a preferred answer set $S$ 
of $\Pi(K,C,D)$ such that $inst(aux_C, D) \not \in S$.  

This problem can be solved by an algorithm that non-deterministically guesses a ground interpretation $S$ over the language of $\Pi(K,C,D)$, 
of polynomial size (in the size of $\Pi(K,C,D)$)
and, then, verifies that $S$ satisfies all rules in $\Pi(K,C,D)$ and is supported in $S$ (i.e., it is an answer set of $\Pi(K,C,D)$),  that $inst(aux_C, D) \not \in S$ and that 
$S$ is preferred among the answer sets of $\Pi(K,C,D)$.
The last point can be verified using an NP-oracle which answers "yes" when $S$ is  a preferred answer set of $\Pi(K,C,D)$, and "no" otherwise.

The oracle checks if there is an answer set $S'$ of $\Pi(K,C,D)$ which is preferred to $S$, by
 non-deterministically guessing a ground polynomial interpretation $S'$ over the language of $\Pi(K,C,D)$,  and verifying that $S$ satisfies all rules and is supported in $S'$ (i.e., it is an answer set of $\Pi(K,C,D)$),  and that $S'$ is preferred to $S$. These checks can be done in polynomial time. 

Hence, deciding whether $\tip(C) \sqsubseteq D$ is not cw$^m$-entailed by $K$ is in $\Sigma^p_2$,  and the complementary problem of deciding cw$^m$-entailment is in $\Pi^p_2$. 
\end{proof}

\begin{proposition} \label{prop:hardness}
The problem of deciding cw$^m$-entailment  is 
\textsc{$\Pi^P_2$}-hard.
\end{proposition}
\begin{proof}
We prove that  the problem of deciding cw$^m$-entailment  is \textsc{$\Pi^P_2$}-hard, by providing a reduction of the minimal entailment problem of
{\em positive disjunctive logic programs}, which has been proved to be a
 \textsc{$\Pi^P_2$}-hard problem by Eiter and Gottlob \shortcite{Eiter95}.

Let $PV = \{p_1, \ldots, p_n \}$ be a set of propositional variables.
A clause is formula $l_1 \vee \ldots \vee l_h$, where each literal $l_j$  is either a
propositional variable $p_i$ or its negation $\neg p_i$.
A positive disjunctive logic program (PDLP) is a set of clauses $S= \{\gamma_1, \ldots, \gamma_m\}$,
where each $\gamma_j$ contains at least one positive literal.
A truth valuation for $S$ is a set $I \subseteq PV$,
containing the propositional variables which are true. A truth valuation is a model of $S$
if it satisfies all clauses in $S$.
For a literal $l$, we write $S \models_{min} l$ if and only if every minimal model (with respect to subset inclusion)
of $S$ satisfies $l$. The minimal-entailment problem can be then defined as follows:
given a PDLP $S$ and a literal $l$, determine whether  $S \models_{min} l$.
In the following we sketch the reduction of the minimal-entailment problem for a  PDLP $S$
to cw$^m$-entailment of an inclusion from a knowledge base $K$ constructed from $S$. 

We define a ranked $\elpb$ knowledge base 
$K=\langle  {\cal T}_{strict},$ $ {\cal T}_{M_1}, \ldots,$ $ {\cal T}_{M_n}, {\cal A}  \rangle$ over the set of concepts ${\cal C}= \{ M_1, \ldots, M_n \}$,
where $\mathit{ABox}= \emptyset $ and each $M_h$ is associated with a propositional variable $p_h \in PV$. 
For each variable $p_h\in PV$ ($h=1,\ldots,n$), we also introduce two concept names $P_h$ and $\overline{P_h}$ in $N_C$, where $\overline{P_h}$ is intended to represent the negation of $p_h$.
We further introduce in $N_C$ an auxiliary concept $H$, 
a concept name $D_S$ associated with the set of clauses $S$,
and a concept name $D_j$ associated with each clause $\gamma_j$ in $S$, for $j=1,\ldots,m$.
$ {\cal T}_{strict}$ contains the following strict inclusions
(where $ C_i^j$ and $\overline{ C_i^j}$ are concepts associated with each literal $l_i^j$ occurring in $\gamma_j = l_1^j \vee \ldots \vee l_k^j$, as defined below):

(1) $C_i^j \sqsubseteq D_j$ \ \ \ for all $\gamma_j = l_1^j \vee \ldots \vee l_k^j$ in $S$

(2) $ D_j \sqcap \overline{ C_1^j}  \sqcap \ldots \sqcap \overline{ C_k^j} \sqsubseteq \bot$
\ \ \ for all $\gamma_j = l_1^j \vee \ldots \vee l_k^j$ in $S$

(3) $D_1  \sqcap \ldots  \sqcap D_m \sqsubseteq D_S$  

(4) $D_S \sqsubseteq  D_1  \sqcap \ldots  \sqcap D_m$  

(5) $H \sqcap P_h \sqcap \overline{ P_h} \sqsubseteq \bot$  \ \ \ ($h=1,\ldots,n$)

(6) $H \sqsubseteq D_S$,

\noindent
for $i=1,\ldots,k$,
 $j=1,\ldots,m$,
where $C_i^j$ is defined as follows:
\[
   C_i^j=
\begin{cases}
	P_h \mbox{\ \ \ \ \ \ \ \ \ \ \ \ \ \ \ \ \ \ \ \ \ \ \ \ \ \  if } l_i^j=p_h\\
    	\overline{P_h} \mbox{\ \ \ \ \ \ \ \ \ \ \ \ \ \ \ \ \ \ \ \ \ \ \ \ \ \ if } l_i^j=\neg p_h
\end{cases}
\]

\[
   \overline{C_i^j}=
\begin{cases}
	\overline{P_h} \mbox{\ \ \ \ \ \ \ \ \ \ \ \ \ \ \ \ \ \ \ \ \ \ \ \ \ \ if } l_i^j=p_h\\
    	P_h \mbox{\ \ \ \ \ \ \ \ \ \ \ \ \ \ \ \ \ \ \ \ \ \ \ \ \ \  if }  l_i^j=\neg p_h
\end{cases}
\]

\noindent
$H$-elements are intended to represent the propositional interpretations.
Inclusions (1)-(4) bind the truth values of the $P_h$'s to the truth values of the clauses in $S$ and of their conjunction (represented by $D_S$). 
By (6), in any cw$^m$-model of $K$, each instance $x$ of $H$ is an instance of $D_S$. 
By (5), each instance $x$ of $H$  cannot be an instance of both $P_h$ and $\overline{P_h}$.
To capture the requirement that an $H$-element must be either an instance of  $P_h$ or  an instance of $\overline{P_h}$, 
we let, for each distinguished concept name $M_h \in {\cal C} $ ($h=1,\ldots,n$), 
 ${\cal T}_{M_h}$ contains the following two typicality inclusions: 
 
$\tip(M_h) \sqsubseteq \overline{P_h}$,  with rank $0$, 

$\tip(M_h) \sqsubseteq P_h$,  with rank $1$. 

\noindent
Relation $<_{M_h}$ prefers  $\overline{P_h}$-elements wrt $P_h$-elements, and prefers $P_h$-elements wrt elements $x$ 
which are neither instances of $P_h$ nor instances of $\overline{P_h}$.
This allows to capture subset inclusion minimality in the interpretations of the positive disjunctive logic program $S$.

It can be proved that the instances of $H$ which are minimal wrt $<_{M_h}$ in a canonical $\tip$-compliant cw$^m$-model of $K$, must be either instances of  $P_h$ or instances of $\overline{P_h}$. 
In particular,  if there is an  $H$-element $x$ which is neither an instance of $P_h$ nor an instance of $\overline{P_h}$, it cannot be $<_{M_h}$-minimal. In fact, in such a case, the canonical model would contain another $H$-element $y$, which is an instance of all the concepts of which $x$ is instance and, in addition, is an instance of $\overline{P_h}$. Clearly $y$ would be preferred to $x$ wrt  $<_{M_h}$, 
contradicting the minimality of $x$ wrt $<_{M_h}$.

It can be proved that minimal $H$-elements with respect to $<$ in a canonical $\tip$-compliant cw$^m$-model of $K$ correspond to minimal models of $S$
and that, given a set $S$ of clauses and a literal $l$,
$$S \models_{min}  l \mbox{\ \ \ } \Leftrightarrow \mbox{\ \ \ } K \models_{cw^m} \tip(H) \sqsubseteq C_l$$
where $C_l$ is the concept associated with $l$, i.e., $C_l=P_h$ if $l=p_h$, and
$C_l=\overline{ P_h}$ if $l=\neg p_h$.

From the reduction above and the result that minimal entailment for PDLP is
 \textsc{$\Pi^P_2$}-hard \cite{Eiter95}, it follows that
 cw$^m$-entailment  is $\Pi^P_2$-hard.
\end{proof}

\end{appendix}

\end{document}